\documentclass[twoside,11pt]{article}

\usepackage{url}
\usepackage{amssymb}
\usepackage{comment}
\usepackage{xspace}
\usepackage{framed}

\newcommand{\reals}{\mathbb{R}}
\newcommand{\eps}{\varepsilon}
\newcommand{\dst}{\displaystyle}
\newcommand{\Ac}[1]{A^{\text{c}}[#1]}
\newcommand{\Agc}[1]{A^{\text{gc}}[#1]}
\newcommand{\Ad}[1]{A^{\text{d}}[#1]}
\newcommand{\Agd}[1]{A^{\text{gd}}[#1]}

\newcommand{\Aic}{A^{\text{c}}_i}
\newcommand{\Aigc}{A^{\text{gc}}_i}
\newcommand{\Bgc}{B^{\text{gc}}}

\newcommand{\Alg}{{\textsc{GreeDi}}\xspace}

\newcommand{\BBalg}{X}
\newcommand{\matroid}{\mathcal{M}}
\newcommand{\indepset}{\mathcal{I}}
\newcommand{\C}{\zeta} 
\newcommand{\optsol}{\Ac{\C}}

\newcommand{\Oa}[1]{A_{\max}^{\text{gc}}[#1]}
\newcommand{\Ob}[1]{A_\textit{B}^{\text{gc}}[#1]}

\usepackage{jmlr2e}

   \graphicspath{{../pdf/}{../jpeg/}{}}
   \DeclareGraphicsExtensions{.pdf,.jpeg,.png}
   
\usepackage{caption}
\usepackage{subcaption}

\usepackage{amsmath,algorithm,algorithmic}

\newcommand{\abs}[1]{\left\lvert#1\right\rvert}

\jmlrheading{16}{2015}{1-41}{10/14}{00/00}{Baharan Mirzasoleiman, Amin Karbasi, Rik Sarkar and Andreas Krause}

\ShortHeadings{Distributed Submodular Maximization}{Mirzasoleiman, Karbasi, Sarkar and Krause}
\firstpageno{1}

\begin{document}

\title{Distributed Submodular Maximization}

\author{\name Baharan Mirzasoleiman \email baharanm@inf.ethz.ch \\
       \addr Department of Computer Science\\
       ETH Zurich\\
       Universitaetstrasse 6, 8092 Zurich, Switzerland
		\AND
       \name Amin Karbasi \email amin.karbasi@yale.edu \\
       \addr School of Engineering and Applied Science\\
       Yale University\\
       New Haven, USA
       \AND
       \name Rik Sarkar \email rsarkar@inf.ed.ac.uk \\
       \addr Department of Informatics \\
       University of Edinburgh \\
		10 Crichton St, Edinburgh EH8 9AB, United Kingdom
       \AND
       \name Andreas Krause \email krausea@ethz.ch \\
       \addr Department of Computer Science\\
       ETH Zurich\\
       Universitaetstrasse 6, 8092 Zurich, Switzerland}

\editor{name}

\maketitle

\begin{abstract}
Many large-scale machine learning problems--clustering, non-parametric learning, kernel machines, etc.--require selecting a small yet representative subset from a  large dataset. Such problems can often be reduced to maximizing a submodular set function subject to various constraints. Classical approaches to submodular optimization require centralized access to the full dataset, which is impractical for truly large-scale problems.  In this paper, we consider the problem of submodular function maximization in a distributed fashion. 
We develop a simple, two-stage protocol \Alg, 
that is easily implemented using MapReduce style computations. We theoretically analyze our approach, and show that under certain natural conditions, performance close to the centralized approach can be achieved. 
We begin with monotone submodular maximization subject to a cardinality constraint, and then extend this approach to obtain approximation guarantees for (not necessarily monotone) submodular maximization subject to more general constraints including matroid or knapsack constraints.
In our extensive experiments, we demonstrate the effectiveness of our approach on several applications, including sparse Gaussian process inference and exemplar based clustering on tens of millions of examples using Hadoop.
\end{abstract}


\section{Introduction}\label{sec:intro}
Numerous machine learning tasks require selecting representative subsets of manageable size out of large datasets. Examples range from exemplar based clustering \citep{dueck07} to active set selection for non-parametric learning \citep{rasmussen06}, to 
viral marketing \citep{kempe03}, and data subset selection for the purpose of training complex models \citep{lin2011class}. Many such problems can be reduced to the problem of {\em maximizing a submodular set function} subject to cardinality or other feasibility constraints such as matroid, or knapsack constraints \citep{krause2010budgeted,krause12survey, lee2009non}.

Submodular functions exhibit a natural diminishing returns property common in many well known objectives: the marginal benefit of any given element decreases as we select more and more elements. Functions such as entropy or maximum weighted coverage are typical examples of functions with diminishing returns. As a result, submodular function optimization has numerous applications in machine learning and social networks: viral marketing \citep{kempe03, babaei2013revenue, mirzasoleiman2012immunizing}, information gathering~\citep{krause11submodularity}, document summarization~\citep{lin2011class}, and active learning~\citep{golovin11jair,guillory2011-active-semisupervised-submodular}.

Although maximizing a submodular function is NP-hard in general,
a seminal result of \cite{nemhauser78} states that a simple greedy algorithm 
produces solutions competitive with the optimal (intractable) solution \citep{nemhauser78best,feige98threshold}. 
However, such greedy algorithms or their accelerated variants \citep{minoux78accelerated,badanidiyuru2014fast,mirzasoleiman2015lazier} do not scale well when the dataset is massive. As data volumes in modern applications increase faster than the ability of individual computers to process them, we need to look at ways to adapt our computations using parallelism. 
 
MapReduce \citep{dean04mapreduce} 
is arguably one of the most successful programming models for 
reliable and efficient  parallel computing. It works by distributing the data to independent machines: {\em map} tasks redistribute the data for appropriate parallel processing and the output then gets sorted and processed in parallel by
{\em reduce} tasks.

To perform submodular optimization in MapReduce, we need to design suitable parallel algorithms. The greedy algorithms  that work well for centralized submodular optimization do not translate easily to parallel environments. The algorithms are inherently sequential in nature, since the marginal gain from adding each element is dependent on the elements picked in previous iterations.   
This mismatch makes it inefficient to apply classical algorithms directly to parallel setups.

In this paper, we develop a distributed procedure for maximizing  submodular functions, that can be easily implemented in MapReduce. Our strategy is to partition the data (e.g., randomly) and process it in parallel. In particular:
\begin{itemize}
\item
We present a simple, parallel protocol, called \Alg for distributed submodular maximization subject to cardinality constraints. 
It requires minimal communication, and can be easily implemented in MapReduce style parallel computation models.

\item 
We show that under some natural conditions, for large datasets the quality of the obtained solution is provably competitive with the best centralized solution.

\item
We discuss extensions of our approach to obtain approximation algorithms for (not-necessarily monotone) submodular maximization subject to more general types of constraints, including matroid and knapsack constraints.

\item
We implement our approach for exemplar based clustering and active set selection in Hadoop, and show how our approach allows to scale exemplar based clustering and sparse Gaussian process inference to datasets containing tens of millions of points.

\item
We extensively evaluate our algorithm 
on several machine learning problems, including
exemplar based clustering, active set selection and finding cuts in graphs, and show that our approach leads to parallel solutions that are very competitive with those obtained via centralized methods ($98\%$  in exemplar based clustering, $97\%$ in  active set selection, $90\%$ in finding cuts).  
\end{itemize}

This paper is organized as follows. We begin in Section \ref{sec:related} by discussing background and related work. In Section \ref{sec:problem}, we formalize the distributed submodular maximization problem under cardinality constraints, and introduce example applications as well as naive approaches toward solving the problem. We subsequently present our \Alg algorithm in Section \ref{sec:results}, and prove its approximation guarantees. We then consider maximizing a submodular function subject to more general constraints in Section \ref{sec:constraints}. We also present computational experiments on very large datasets in Section \ref{sec:experiments}, showing that in addition to its provable approximation guarantees, our algorithm provides results close to the centralized greedy algorithm. We conclude in Section \ref{sec:conclusion}.

\section{Background and Related Work}\label{sec:related}
\looseness= -1
\subsection{Distributed Data Analysis and MapReduce} 
Due to the rapid increase in dataset sizes, and the relatively slow advances in sequential processing capabilities of modern CPUs, parallel computing paradigms have received much interest. Inhabiting a sweet spot of resiliency, expressivity and programming ease, the MapReduce style computing model ~\citep{dean04mapreduce} has emerged as prominent foundation for large scale machine learning and data mining algorithms~\citep{chu06map, ekanayake08}.
A MapReduce job takes the input data as a set of \textit{$<key; value>$} pairs. Each job consists of three stages: the \textit{map} stage, the \textit{shuffle} stage, and the \textit{reduce} stage. The map stage, partitions the data randomly across a number of machines by associating each element with a \textit{key} and produce a set of $<key; value>$ pairs. Then, in the shuffle stage, the value associated with all of the elements with the same key gets merged and send to the same machine.
Each reducer then processes the values associated with the same key and outputs a set of new $<key; value>$ pairs with the same key. The reducers' output could be input to another MapReduce job and a program in MapReduce paradigm can consist of multiple rounds of map and reduce stages \citep{karloff2010model}.
\subsection{Centralized and Streaming Submodular Maximization} 
The problem of centralized maximization of submodular functions has received much interest, starting with the seminal work of \cite{nemhauser78}. Recent work has focused on providing approximation guarantees for more complex constraints \cite[for a more detailed account, see the recent survey by][]{krause12survey}. \citet{golovin10distributed} consider an algorithm for online distributed submodular maximization with an application to sensor selection. However, their approach requires $k$ stages of communication, which is unrealistic for large $k$ in a MapReduce style model. \citet{krause2010budgeted} consider the problem of submodular maximization in a streaming model; however, their approach makes strong assumptions about the way the data stream is generated and is not applicable to the general distributed setting. 
Recently, \citet{Badanidiyuru2014Streaming} provide a single pass streaming algorithm for cardinality-constrained submodular maximization with $1/2-\varepsilon$ approximation guarantee to the optimum solution that makes no assumptions on the data stream.

There has also been new improvements in the running time of the standard greedy solution for solving SET-COVER (a special case of submodular maximization) when the data is large and disk resident \citep{cormode10}. More generally, \citet{badanidiyuru2014fast} and \citet{mirzasoleiman2015lazier} improve the running time of the greedy algorithm for maximizing a monotone submodular function by reducing the number of oracle calls to the objective function.  In a similar spirit,~\citet{wei2014fast} propose a multi-stage  framework for submodular maximization. In order to reduce the memory and 
computation cost, they apply an approximate greedy procedure to maximize surrogate (proxy) submodular functions  instead of optimizing the target function at each stage.  
 The above approaches are sequential in nature and it is not clear how to parallelize them.  However, they  can be naturally integrated into our distributed framework to achieve further acceleration. 
\subsection{Scaling Up: Distributed Algorithms}  Recent work has focused on specific instances of submodular optimization in distributed settings.
Such scenarios often occur in large-scale graph mining problems where the data itself is too large to be stored on one machine.
In particular, \citet{Chierichetti2010} address the MAX-COVER problem and provide a $(1-1/e-\epsilon)$ approximation to the centralized algorithm at the cost of passing over the dataset many times. Their  result is further improved by 
\citet{Blelloch2011spaa}.  
\cite{lattanzi11} address more general graph problems by introducing the idea of filtering, namely, reducing the size of the input in a distributed fashion
so that the resulting, much smaller, problem instance can be solved
on a single machine. This idea is, in spirit, similar to our distributed method \Alg.  In contrast,  we provide a more general framework, and characterize settings where performance competitive with the centralized setting can be obtained. The present version is a significant extension of our previous conference paper \citep{mirzasoleiman2013distributed}, providing theoretical guarantees for both monotone and non-monotone submodular maximization problems subject to more general types of constraints, including matroid and knapsack constraints (described in Section \ref{sec:constraints}), 
and additional empirical results (Section \ref{sec:experiments}).  
Parallel to our efforts \citep{mirzasoleiman2013distributed},~\cite{kumar2013fast} has taken the approach of adapting the sequential greedy algorithm to distributed settings. However, their method requires knowledge of the ratio between the largest and smallest marginal gains of the elements, and generally requires a non-constant (logarithmic) number of rounds. We provide empirical comparisons in Section~\ref{sec:comparision}.

\section{Submodular Maximization}\label{sec:problem}
In this section, we first review  submodular functions and how to greedily maximize them. We then describe the \textit{distributed submodular maximization} problem, the focus of this paper. Finally, we  discuss two naive approaches towards solving this problem.

\subsection{Greedy Submodular Maximization}
Suppose that we have a large dataset of images, e.g. the set of all images on the Web or an online image hosting website such as Flickr, and we wish to  retrieve a subset of images that best represents the visual appearance  of the dataset. Collectively, these images can be considered as \textit{exemplars} that \textit{summarize} the visual categories
of the dataset as shown in Fig. \ref{fig:img}.

\begin{figure}[t]
                \centering
                \includegraphics[width=1\textwidth]{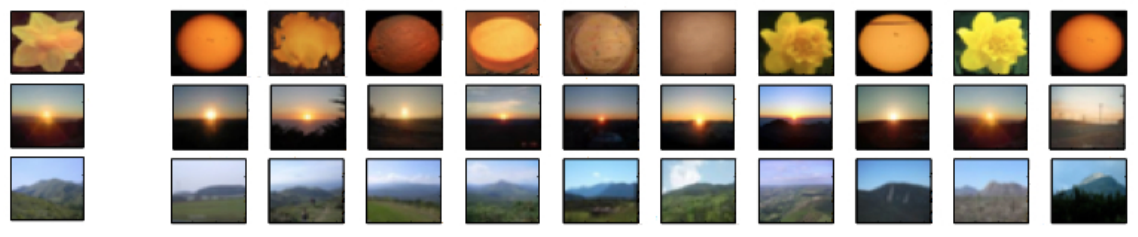}
                \captionsetup{format=hang}
                \caption{Cluster exemplars (left column) discovered by our distributed algorithm GreeDi described in Section \ref{sec:results} 
                applied to the Tiny Images dataset \citep{torralba200880}, and a set of representatives from each cluster.}
                \label{fig:img}
\end{figure}

One way to approach this problem is to formalize it as the \textit{$k$-medoid} problem. Given a set $V=\{e_1, e_2,\dots, e_n\}$ of images (called ground set) associated with a (not necessarily symmetric) dissimilarity function, we seek to select a subset $S\subseteq V$ of at most $k$  exemplars or cluster centers, and then assign each image in the dataset to its least dissimilar exemplar. 
If an element $e \in V$ is assigned to exemplar $v\in S$, then the cost associated with $e$ is the dissimilarity between $e$ and $v$. 
The goal of the $k$-medoid problem is to choose exemplars that minimize the sum of dissimilarities between every data point $e \in V$ and its assigned cluster center.

Solving the $k$-medoid problem optimally is NP-hard, however, as we discuss in Section \ref{sec:examples}, we can transform this problem, and many other summarization tasks, to the problem of maximizing a monotone submodular function subject to a cardinality constraint

\begin{equation} \label{problem1}
\max_{S\subseteq V} f(S)  \quad \text{s.t.}  \quad |S|\leq k.
\end{equation} 

\noindent Submodular functions are set functions 
which satisfy 
the following natural diminishing returns property.  

\begin{definition}[c.f., \citet{nemhauser78}]
A set function $f: 2^V \rightarrow \mathbb{R}$ is {\em submodular}, if for every $A\subseteq B \subseteq V$ and $e\in V\setminus B$
\[ f(A \cup \{ e \}) - f(A) \geq f(B \cup \{ e \}) - f(B). \]
Furthermore, $f$ is called {\em monotone} iff for all $A \subseteq B \subseteq V$ it holds that $f(A) \leq f(B)$.
\end{definition}

\noindent 
We will generally additionally require that $f$ is nonnegative, i.e., $f(A)\geq 0$ for all sets $A$.

Problem~\eqref{problem1} is NP-hard for many classes of submodular functions~\citep{feige98threshold}. A fundamental result by ~\cite{nemhauser78} establishes that a simple greedy algorithm that starts with the empty set  and iteratively augments the current solution with an element of maximum incremental value  
\begin{equation}\label{eq:argmax}
v^* = \arg\max_{v \in V \setminus A} f(A \cup \{ v \}),
\end{equation}
continuing until $k$ elements have been selected, is guaranteed to provide a constant factor approximation.

\begin{theorem} \textbf{\citep{nemhauser78}}
For any non-negative and monotone submodular function $f$, the greedy heuristic always produces a solution $\Agc{k}$ of size $k$ that achieves at least a constant factor $(1-1/e)$ of the optimal solution.
\[  f(\Agc{k}) \geq (1-1/e) \max_{|A| \leq k} f(A). \]
\end{theorem}
\noindent This result can be easily extended to
$  f(\Agc{l}) \geq (1-e^{-l/k}) \max_{|A| \leq k} f(A) $, where $l$ and $k$ are two positive integers
\cite[see,][]{krause12survey}. 

\subsection{Distributed Submodular Maximization}\label{sec:dsmproblem}

In many today's applications where the size of the ground set $|V|=n$ 
 is very large and cannot be stored on a single computer, running the standard greedy algorithm or its variants  \cite[e.g., lazy evaluations,][]{minoux78accelerated, leskovec07, mirzasoleiman2015lazier} in a centralized manner is infeasible. Hence, we seek a solution that is suitable for large-scale parallel computation. 
The greedy method described above is in general difficult to parallelize, since it is inherently sequential: at each
step, only the object with the highest marginal gain is chosen and every subsequent step  depends on the preceding ones.

Concretely, we consider the setting where the ground set $V$  is very large and cannot be handled on a single machine,  thus must be distributed among a set of $m$ machines. While there are several approaches towards parallel computation, in this paper we consider the following model that can be naturally implemented in  MapReduce. The computation proceeds in a sequence of rounds. In each round, the dataset is distributed to $m$ machines. Each machine $i$ carries out computations independently in parallel on its local data. After all machines finish, they synchronize by exchanging a limited amount of data (of size polynomial in $k$ and $m$, but independent of $n$). Hence, any distributed algorithm in this model must specify: 
1) how to distribute $V$ among the $m$ machines, 2) which algorithm should run on each machine, and 3) how to communicate and merge the resulting solutions. 

In particular, the distributed submodular maximization problem requires the specification of the above steps in order to implement an approach for submodular maximization. More precisely, given a monotone submodular function $f$, a cardinality constraint $k$, and a number of machines $m$, we wish to produce a solution $\Ad{m,k}$ of size $k$ such that $f(\Ad{m,k})$ is competitive with the optimal {\em centralized} solution $\max_{|A|\leq k, A\subseteq V} f(A)$.

 \subsection{Naive Approaches Towards Distributed Submodular Maximization}
One way to solve problem~\eqref{problem1} in a distributed fashion is as follows. 
The dataset is first partitioned (randomly, or using some other strategy) onto the $m$ machines,
with $V_i$ representing the data allocated to machine $i$. We then proceed in $k$ rounds. In each round, all machines--in parallel--compute the marginal gains of all elements in their sets $V_i$. Next, they  communicate their candidate to a central processor, who identifies the globally best element, which is in turn communicated to the $m$ machines. This element is then taken into account for computing the marginal gains and selecting the next  elements. This algorithm (up to decisions on how break ties)  implements exactly the centralized greedy algorithm, and hence provides the same approximation guarantees on the quality of the solution.
Unfortunately, this approach requires synchronization after each of the $k$ rounds. In many applications, $k$ is quite large (e.g., tens of thousands), rendering this approach impractical for MapReduce style computations. 

An alternative approach for large $k$ would be to   greedily select $k/m$ elements independently  on each machine (without synchronization), and then merge them to obtain a solution of size $k$. This approach that requires only two rounds (as opposed to $k$), is much more communication efficient, and can be easily implemented using a single MapReduce stage. Unfortunately, many machines may select redundant elements, and thus the merged solution may suffer from diminishing returns. It is not hard to construct examples for which  this approach  produces solutions that are a factor $\Omega(m)$ worse than the centralized solution. 

In Section~\ref{sec:results}, we introduce an alternative protocol \Alg, which requires little communication, while at the same time yielding a solution competitive with the centralized one, under certain natural additional assumptions.

\subsection{Applications of Distributed Submodular Maximization}\label{sec:examples}

In this part, we discuss two concrete problem instances, with their  corresponding submodular objective functions $f$, where the size of the datasets often  requires a distributed solution for the underlying submodular maximization.
\subsubsection{Large-scale Nonparametric Learning} \label{sec:nonparam}
Nonparametric learning (i.e., learning of models whose complexity may depend on the dataset size $n$) are notoriously hard to scale to large datasets. A concrete instance of this problem arises from training Gaussian processes or performing MAP inference in Determinantal Point Processes, as considered below. Similar challenges arise in many related  learning methods, such as training kernel machines, when attempting to scale them to large data sets.

\textit{Active Set Selection in Sparse Gaussian Processes (GPs).} Formally a GP is a joint probability distribution over a (possibly infinite) set of random variables $\mathbf{X}_{V}$, indexed by the ground set $V$, such that every (finite) subset $\mathbf{X}_S$ for $S=\{e_1, \dots, e_s\}$ is distributed according to a multivariate normal distribution. More precisely,  we have $$P(\mathbf{X}_S = \mathbf{x}_S) = \mathcal{N}(\mathbf{X}_S; \mu_S, \Sigma_{S,S}),$$ where $\mu= (\mu_{e_1}, \dots, \mu_{e_s})$ and $\Sigma_{S,S} = [\mathcal{K}_{e_i,e_j}]$ are prior mean and  covariance matrix, respectively.
The covariance matrix is parametrized via a positive definite kernel $\mathcal{K(\cdot,\cdot)}$. As a concrete example, when elements of the ground set $V$ are embedded in a Euclidean space, a commonly used kernel  in practice is the squared exponential kernel defined as follows:
%
$$\mathcal{K}(e_i,e_j) = \exp(-||e_i-e_j||^2_2/h^2).$$
\noindent Gaussian processes are commonly used as priors for nonparametric regression. In GP regression, each data point $e\in V$ is considered a random variable. Upon observations $\mathbf{y}_A = \mathbf{x}_A+\mathbf{n}_A$ (where $\mathbf{n}_A$ is a vector of independent Gaussian noise with variance $\sigma^2$), the predictive distribution of a new data point $e\in V$ is a normal distribution $P(\mathbf{X}_e\mid \mathbf{y}_A) = \mathcal{N}(\mu_{e|A}, \Sigma^2_{e|A})$, where mean $\mu_{e|A}$ and variance $\sigma^2_{e|A}$ are given by
\begin{align}
& \mu_{e|A} = \mu_e + \Sigma_{e,A} (\Sigma_{A,A}+\sigma^2\mathbf{I})^{-1} (\mathbf{x}_A-\mu_A), \label{eq:GPmean}\\
& \sigma^2_{e|A} = \sigma^2_e - \Sigma_{e,A} (\Sigma_{A,A}+\sigma^2\mathbf{I})^{-1}\Sigma_{A,e}. \label{eq:GPvar}
\end{align}

Evaluating \eqref{eq:GPmean} and~\eqref{eq:GPvar} is computationally expensive as it requires solving a linear system of $|A|$ variables. Instead, most efficient approaches for making predictions in GPs rely on choosing a small--so called \textit{active}--set of data points. For instance, in the Informative Vector Machine (IVM) one seeks a set $S$ such that the \textit{information gain}, defined as 
$$f(S) = I (\mathbf{Y}_S;\mathbf{X}_{V}) = H(\mathbf{X}_{V}) - H(\mathbf{X}_{V}|\mathbf{Y}_{S}) = \frac{1}{2}\log\det(\mathbf{I}+\sigma^{-2}\Sigma_{S,S})$$ is maximized. 
\noindent It can be shown that this choice of $f$ is monotone submodular \citep{krause05near}. For medium-scale problems, the standard greedy algorithms provide good solutions. For massive data however, we need to resort to distributed algorithms. In Section~\ref{sec:experiments}, we will show how \Alg can choose near-optimal subsets out of a dataset of 45 million vectors.

\textit{Inference for Determinantal Point Processes.}
A very similar problem arises when performing inference in Determinantal Point Processes (DPPs). DPPs \citep{macchi1975coincidence} are distributions over subsets with a preference for diversity, i.e., there is a higher probability associated with sets  containing dissimilar elements. Formally, a point process $\mathcal{P}$ on a set of items $V = \{1,2, ..., N \}$ is a probability measure on $2^V$ (the set of all subsets of $V$). $\mathcal{P}$ is called {\em determinantal point process} if for every $S \subseteq V$ we have: 
$$\mathcal{P}(S) \propto \det (K_S),$$
where $K$ is a positive semidefinite kernel matrix, and $K_S \equiv [K_{ij}]_{i,j \in S}$, is the restriction of $K$ to the entries indexed by elements of $S$ (we adopt that $\det (K_{\emptyset}) = 1$).
The normalization constant can be computed explicitly from the following equation
$$\sum_S \det(K_S ) = \det(\mathbf{I} + K),$$
where $I$ is the $N \times N$ identity matrix. Intuitively, the kernel matrix determines which items are similar and therefore less likely to appear together. 

In order to find the most diverse and informative subset of size $k$, we need to find $\arg \max_{|S|\leq k} \det(K_S)$ which is NP-hard, as the total number of possible subsets is exponential \citep{ko1995exact}. However, the objective function is log-submodular, i.e. $f(S)=\log \det(K_S )$ is a submodular function \citep{kulesza2012determinantal}. Hence, MAP inference in large DPPs is another potential application of distributed submodular maximization.

\subsubsection{Large-scale Exemplar Based Clustering} 

Suppose we wish to select a set of exemplars, that best represent a massive dataset.
One approach for finding such exemplars is solving the $k$-medoid problem \citep{kaufman2009finding}, which aims to minimize the sum of pairwise dissimilarities between exemplars and elements of the dataset. 
More precisely, let us assume that for the dataset $V$ we are given a nonnegative function $l:V\times V\rightarrow \mathbb{R}$ (not necessarily assumed symmetric, nor obeying the triangle inequality) such that  $l (\cdot, \cdot)$ encodes dissimilarity between elements of the underlying set $V$.  Then, the cost function for the $k$-medoid problem is: 

\begin{equation}\label{eq:loss_definition}
L(S) = \frac{1}{|V|}\sum_{v\in V} \min_{e \in S} l(e,\upsilon). 
\end{equation}
Finding the subset $$S^* = \arg\min_{|S| \leq k} L(S)$$ of cardinality at most $k$ that minimizes the cost function \eqref{eq:loss_definition} is NP-hard. 
However, by introducing an auxiliary element $e_0$, a.k.a. phantom exemplar, we can turn $L$ into a monotone submodular function \citep{krause2010budgeted} 
\begin{equation}\label{eq:loss_red}
f(S) = L(\{e_0\}) - L(S\cup \{e_0\}).
\end{equation}
\noindent In  words, $f$ measures the decrease in the loss associated with the  set $S$ versus the loss associated with just the auxiliary element.  
We begin with a phantom exemplar and try to find the active set that together with the phantom exemplar 
reduces the value of our loss function more than any other set. Technically, any point $e_0$ that satisfies the following condition can be used as a phantom exemplar:
$$\max_{v' \in V}l(v,v') \leq l(v,e_0), \quad \forall v \in V \setminus S.$$
This condition ensures that once the distance between any $v \in V \setminus S$ and $e_0$ is greater than the maximum distance between elements in the dataset, then $L(S\cup \{e_0\}) = L(S)$. As a result, maximizing $f$ (a monotone submodular function)  is equivalent to minimizing the cost function $L$. 
This problem becomes especially computationally challenging when we have a large dataset and we wish to extract a manageable-size set of exemplars, further motivating our distributed approach.

\subsubsection{Other Examples}
Numerous other real world problems in machine learning 
can be modeled as maximizing a monotone submodular function subject to appropriate constraints (e.g., cardinality, matroid, knapsack).  To name a few, specific applications that have been considered range from efficient content discovery for web crawlers and multi topic blog-watch \citep{Chierichetti2010}, over document summarization \citep{lin2011class} and speech data subset selection \citep{wei2013using}, to outbreak detection in social networks \citep{leskovec07}, online advertising and network routing \citep{de2003combinatorial}, revenue maximization in social networks \citep{hartline2008optimal}, and inferring network of influence \citep{gomez10}. In all such examples, the size of the dataset (e.g., number of webpages, size of the corpus, number of blogs in the blogosphere, number of nodes in social networks) is  massive, thus \Alg offers a scalable approach, in contrast to the standard greedy algorithm, for such problems.

\section{The {\sc \Alg} Approach for Distributed Submodular Maximization}\label{sec:results}
In this section we present our main results. We first provide our distributed solution \Alg~for  maximizing submodular functions under cardinality constraints. We then show how we can make use of the geometry of data inherent in many practical settings in order to obtain strong data-dependent bounds on the performance of our distributed algorithm.

\subsection{An Intractable, yet Communication Efficient Approach}
Before we introduce \Alg, we first consider an intractable, but communication--efficient two-round parallel protocol to illustrate the ideas. This approach,
shown in Algorithm~\ref{alg:dist}, first distributes the ground set $V$ to $m$ machines. Each machine then finds the {\em optimal} solution, i.e., a set of cardinality at most $k$, that maximizes the value of $f$ in each partition. These solutions are then merged, and the optimal subset of cardinality $k$ is found in the combined set. We denote  this distributed solution by $f(\Ad{m,k})$.

As the optimum centralized solution $\Ac{k}$ achieves the maximum value of the submodular function, it is clear that $f(\Ac{k})\geq f(\Ad{m,k})$.
For the special case of selecting a single element $k=1$, we have  $f(\Ac{1}) = f(\Ad{m,1})$. Furthermore, for {\em modular} functions $f$ (i.e., those for which $f$ and $-f$ are both submodular), it is easy to see that the distributed scheme in fact returns the optimal centralized solution as well.
In general, however, there can be a gap between the distributed and the centralized solution. 
Nonetheless, as the following theorem shows, this gap cannot be more than $1/\min(m,k)$. 
Furthermore, this result is tight. 
\begin{theorem}\label{th:gap}
Let $f$ be a monotone submodular function and let $k>0$. Then, $f(\Ad{m,k})) \geq \frac{1}{\min(m,k)} f(\Ac{k})$. In contrast, for any value of $m$ and $k$, there is a monotone submodular function $f$ such that $f(\Ac{k}) = \min(m,k) \cdot f(\Ad{m,k})$.
\end{theorem}
The proof of all the theorems can be found in the appendix. The above theorem fully characterizes the performance of Algorithm~\ref{alg:dist} in terms of the best centralized solution.  
In practice, we cannot run Algorithm~\ref{alg:dist}, since there is no efficient way to  identify the optimum subset $A_i^c[k]$ in  set $V_i$, unless P=NP. 
In the following, we introduce an efficient distributed approximation -- \Alg. We will further show, that under some additional assumptions, much stronger guarantees can be obtained.

\begin{figure}[t!]
\vspace{-.5cm}
\begin{algorithm}[H]
\caption{Inefficient Distributed Submodular Maximization \vspace{.1cm}} 
\label{alg:dist}
\begin{algorithmic}[1]
\REQUIRE{Set $V$, $\# $of partitions $m$, constraints  $k$.}
\ENSURE{Set $\Ad{m,k}$.}
\STATE Partition $V$ into $m$ sets $V_1, V_2, \dots, V_m$.
\STATE In each partition $V_i$ find the optimum set $A^c_i[k]$ of cardinality $k$.
\STATE Merge the resulting sets: $B = \cup_{i=1}^m A^c_i[k].$
\STATE Find the optimum set of cardinality $k$ in $B$. 
Output this solution $\Ad{m,k}$. 
\end{algorithmic}
\end{algorithm}
\hspace{0.05cm}
\end{figure}

\subsection{Our {\sc \Alg} Approximation}

Our  efficient distributed method \Alg is shown in Algorithm~\ref{alg:gdist}. It parallels the intractable Algorithm~\ref{alg:dist}, but replaces the selection of optimal subsets, i.e., $\Aic[k]$, by greedy solutions $\Aigc[k]$.  Due to the approximate nature of the greedy algorithm, we allow it  to pick sets slightly larger than $k$.
More precisely, \Alg is a two-round algorithm that takes the ground set $V$, the number of partitions $m$, and the cardinality constraint $\kappa$. 
It first distributes the ground set over $m$ machines. Then each machine separately runs the standard greedy algorithm by sequentially finding an element $e\in V_i$ that maximizes the discrete derivative  (\ref{eq:argmax}). Each machine $i$--in parallel--continues adding elements  to the set $\Aigc[\cdot]$  until it reaches $\kappa$ elements. We define  $\Oa{\kappa}$ to be the set with the maximum value among $\{A^{\text{gc}}_1[\kappa], A^{\text{gc}}_2[\kappa], \dots, A^{\text{gc}}_m[\kappa]\}$. Then the solutions are merged, i.e., $B = \cup_{i=1}^m \Aigc[\kappa]$, and another round of greedy selection is performed over $B$ 
until $\kappa$ elements are selected. We denote this solution by $\Ob{\kappa}$. The final distributed solution with parameters $m$ and $\kappa$, denoted by $\Agd{m,\kappa}$, is the set with a higher value   between  $\Oa{\kappa}$ and $\Ob{\kappa}$ (\textit{c.f.}, Figure \ref{fig:GreeDi} shows \Alg schematically).  The following result  parallels Theorem~\ref{th:gap}.
%

\begin{figure}[t!]
\vspace{-.5cm}
\begin{algorithm}[H]
\caption{Greedy Distributed Submodular Maximization (\Alg)} 
\label{alg:gdist}
\begin{algorithmic}[1]
\REQUIRE{Set $V$, $\# $of partitions $m$, constraints $\kappa$.}
\ENSURE{Set $\Agd{m,\kappa}$.}
\STATE Partition $V$ into $m$ sets $V_1, V_2, \dots, V_m$ {(arbitrarily or at random)}.
\STATE Run the standard greedy algorithm on each set $V_i$ to find a solution $\Aigc[\kappa]$. 
\STATE  Find $\Oa{\kappa} = \arg\max_A\{F(A):A\in\{A^{\text{gc}}_1[\kappa],\dots,A^{\text{gc}}_m[\kappa]\}\}$ 
\STATE Merge the resulting sets: $B = \cup_{i=1}^m \Aigc[\kappa].$
\STATE Run the standard greedy algorithm on $B$ to find a solution $\Ob{\kappa}$.
\STATE Return $\Agd{m,\kappa} = \arg\max_A \{F(A): A\in \{\Oa{\kappa}, \Ob{\kappa}\}\}$.
\end{algorithmic}
\end{algorithm}
\hspace{0.05cm}
\end{figure}

\begin{theorem}\label{th:ggreedy}
Let $f$ be a monotone submodular function and  
$\kappa \geq k$. Then $$f(\Agd{m,\kappa}) \geq \frac{(1-e^{-\kappa/k})}{\min(m,k)} f(\Ac{k}).$$ 
\end{theorem}
For the special case of $\kappa=k$ the result of \ref{th:ggreedy} simplifies to $f(\Agd{m,\kappa}) \geq \frac{(1-1/e)}{\min(m,k)} f(\Ac{k})$. Moreover, it is straightforward to generalize \Alg to multiple rounds (i.e., more than two) for very large datasets.

In light of Theorem~\ref{th:gap}, one can expect that in general it is impossible to eliminate the dependency of the distributed solution on $\min(k,m)$\footnote{It has been very recently shown by \cite{mirzasoleiman2015distributed} that the tightest dependency is $\Theta(\sqrt{\min(m,k)})$.}. However, as we show in the sequel, in many practical settings, the ground set $V$ exhibits rich geometrical structure that can be used to obtain stronger guarantees.

\begin{figure}[t]
                \centering
                \includegraphics[width=1\textwidth]{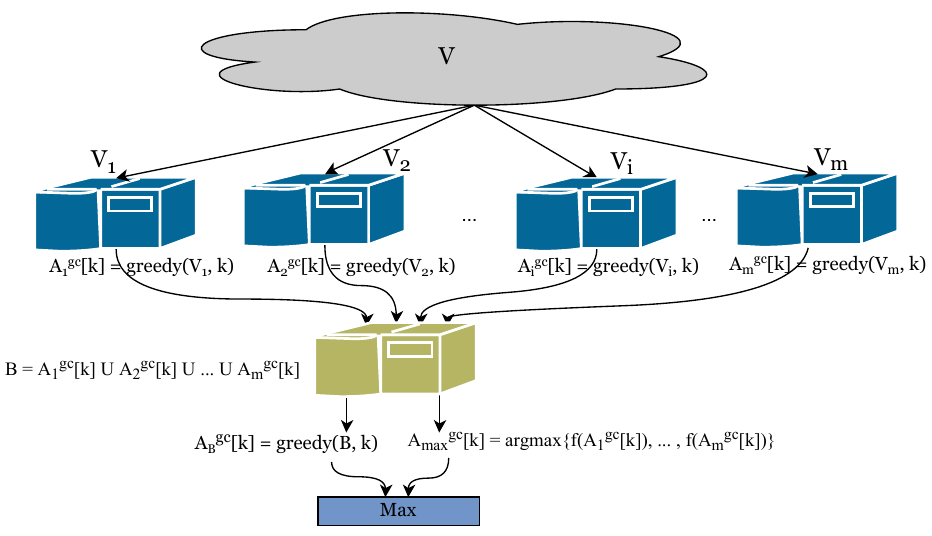}
                \caption{Illustration of our two-round algorithm \Alg
                }
                \label{fig:GreeDi}
\end{figure}

\subsection{Performance on Datasets with Geometric Structure}\label{sec:geo}

In practice, we can hope to do much better than the worst case bounds shown previously by exploiting underlying structure often present in real data and important set functions. In this part, we assume that a metric $d:V\times V\to \mathbb{R}$ exists on the data elements, and analyze performance of the algorithm on functions that vary slowly with changes in the input. We refer to these as {\em Lipschitz functions}:
\begin{definition} Let $\lambda>0$. A set function $f:2^V \to \reals$ is $\lambda$-\textit{Lipschitz} w.r.t.~metric $d$ on $V$, if for any integer $k$, any equal sized sets $S=\{e_1, e_2,\dots, e_k\}\subseteq V$ and $S^\prime=\{e^\prime_1, e^\prime_2,\dots,e^\prime_k\}\subseteq V$ and any matching of elements: $M=\{(e_1, e^\prime_1), (e_2, e^\prime_2)\dots, (e_k, e^\prime_k)\}$, the difference between $f(S)$ and $f(S^\prime)$ is bounded by: 
\begin{equation}
\dst\left| f(S) - f(S^\prime)\right| \leq \lambda\sum_i d(e_i, e^\prime_i).
\end{equation}\label{eq:lip}
\end{definition}
\noindent We can show that the objective functions from both examples in Section~\ref{sec:examples} are $\lambda$-Lipschitz for suitable kernels/distance functions:

\begin{proposition}\label{prop:lipschitz-active}
Suppose that the covariance matrix of a Gaussian process is parametrized via a positive definite kernel $\mathcal{K}: V \times V \rightarrow \mathbb{R}$ which is Lipschitz continuous with respect to metric $d: V \times V \rightarrow \mathbb{R}$ with constant $\mathcal{L}$, i.e., for any triple of points $x_1, x_2, x_3 \in V$, we have $| \mathcal{K}(x_1, x_3) - \mathcal{K}(x_2, x_3) | \leq \mathcal{L} d(x_1, x_2)$.
Then, the mutual information $I(\mathbf{Y}_{S}; \mathbf{X}_{V}) = \frac{1}{2} \log \det(\mathbf{I} + K)$
for the Gaussian process is $\lambda$-Lipschitz  with  
$\lambda = \mathcal{L} k^3$, where $k$ is the number of elements in the selected subset $S$. 
\end{proposition}

\begin{proposition}\label{prop:lipschitz-exemplar}

Let  $d:V\times V\to \mathbb{R}$ be a metric on the elements of the dataset. Furthermore, let  $l : V \times V \rightarrow \mathbb{R}$ encode the dissimilarity between elements of the underlying set $V$. 
Then for $l = d^\alpha$, $\alpha \geq 1$ the loss function $L(S) = \frac{1}{|V|}\sum_{v\in V} \min_{e \in S} l(e,\upsilon)$ (and hence also the corresponding submodular utility function $f$) is $\lambda$-Lipschitz with $\lambda=\alpha R^{\alpha-1}$, where $R$ 
 is the diameter of the ball encompassing elements of the dataset in the metric space.
In particular, for the $k$-medoid problem, which minimizes the loss function over all clusters with respect to $l=d$, we have $\lambda=1$, and for the $k$-means problem, which minimizes the loss function over all clusters with respect to $l=d^2$, we have $\lambda=2R$.
\end{proposition}

Beyond Lipschitz-continuity, many practical instances of submodular maximization can be expected to satisfy a natural \textit{density} condition. Concretely, whenever we consider a representative set (i.e., optimal solution to the submodular maximization problem), we expect that any of its constituent elements has potential candidates for replacement in the ground set. For example, in our exemplar-based clustering application, we expect that cluster centers are not isolated points, but have many almost equally representative points close by. Formally, 
for any element $v\in V$, we define its $\alpha${\em-neighborhood} as the set of elements in $V$ within  distance $\alpha$ from $v$ (i.e., $\alpha$-close to $v$): $$N_\alpha (v) = \{w:d(v,w)\leq \alpha\}.$$ 
By $\lambda$-Lipschitz-continuity, it must hold that if we replace element $v$ in set $S$ by an $\alpha$-close element $v'$ (i.e., $v'\in N_\alpha(v)$) to get a new set $S'$ of equal size, it must hold that $|f(S)-f(S')|\leq \alpha \lambda$.

As described earlier, our algorithm \Alg~partitions $V$ into sets $V_1, V_2,\dots V_m$ for parallel processing. If in addition we assume that elements are assigned uniformly at random to  different machines, $\alpha$-neighborhoods are sufficiently dense, and the submodular function is Lipschitz continuous, then \Alg is guaranteed to produce a solution close to the centralized one. More formally, we have the following theorem.
\begin{theorem}\label{thm:alg-neighborhoods}
Under the conditions that 1) elements are assigned uniformly at random to $m$ machines, 2) for each $e_i\in \Ac{k}$ we have $\abs{N_\alpha(e_i)}\geq k m \log(k/\delta^{1/m})$, and 3) $f$ is $\lambda$-Lipschitz continuous, then with probability at least $(1- \delta)$ the following holds:
%
$$f(\Agd{m,\kappa})\geq (1-e^{-\kappa/k})(f(\Ac{k})-\lambda\alpha k).$$
\end{theorem}

Note that once the above conditions are satisfied for small values of $\alpha$ (meaning that there is a high density of data points within a small distance from each element of the optimal solution) then the distributed solution will be close to the optimal centralized one. In particular if we let $\alpha\rightarrow 0$, the distributed solution is guaranteed to be within a $1-e^{\kappa/k}$ factor from the optimal centralized solution. This situation naturally corresponds to very large datasets. In the  following, we discuss more thoroughly this important scenario.

\subsection{Performance Guarantees for Very Large Datasets}

Suppose that our dataset is a finite sample $V$ drawn i.i.d. from an underlying {\em infinite} set $\mathcal{V}$, according to some (unknown) probability distribution. Let $\Ac{k}$ be an optimal solution in the infinite set, i.e., $\Ac{k} = \arg\max_{S \subseteq \mathcal{V}} f(S)$, such that around each $e_i\in \Ac{k}$, there is a neighborhood of radius at least $\alpha^*$ where the probability density is at least $\beta$ at all points (for some constants $\alpha^*$ and $\beta$). This implies that the solution consists of elements  coming from reasonably dense and therefore representative regions of the dataset.

Let us suppose $g:\reals \rightarrow \reals$ is the {\em growth function of the metric}: $g(\alpha)$ is defined to be the volume of a ball of radius $\alpha$ centered at a point in the metric space. This means, for $e_i\in \Ac{k}$ the probability of a random element being in $N_\alpha(e_i)$ is at least $\beta g(\alpha)$ and the expected number of $\alpha$ neighbors of $e_i$ is at least $E[\abs{N_\alpha(e_i)}]=n\beta g(\alpha)$. 
 As a concrete example, Euclidean metrics of dimension $D$ have $g(\alpha)=O(\alpha^D)$. Note that for simplicity we are assuming the metric to be homogeneous, so that the growth function is the same at every point. For heterogeneous spaces, we require $g$ to have a uniform lower bound on the growth function at every point.

In these circumstances, the following theorem guarantees that if the dataset $V$ is sufficiently large and $f$ is  $\lambda$-Lipschitz, then \Alg~produces a solution close to the centralized one.

\begin{theorem}\label{thm:alg-sampling}
For $\dst n\geq \frac{8k m\log(k/\delta^{1/m})}{\beta  g(\frac{\eps}{\lambda k})}$, where $\frac{\eps}{\lambda k}\leq \alpha^*$, if the algorithm \Alg~assigns elements uniformly randomly to $m$ processors 
, then with probability at least $(1- \delta)$, $$f(\Agd{m,\kappa})\geq (1-e^{-\kappa/k})(f(\Ac{k})-\eps).$$
\end{theorem}

The above theorem shows that for very large datasets, \Alg provides a solution that is within a $1-e^{\kappa/k}$ factor of the optimal centralized solution. This result is based on the fact that for sufficiently large datasets, there is a suitably dense neighborhood around each member of the optimal solution. Thus, if the elements of the dataset are partitioned uniformly randomly to $m$ processors, at least one partition contains a set $A_i^c[k]$ such that its elements are very close to the elements of the optimal centralized solution and provides a constant factor approximation of the optimal centralized solution.

\subsection{Handling Decomposable Functions}\label{sec:extension} 
So far, we have assumed that the objective function $f$ is given to us as a black box, which we can evaluate for any given set $S$ {\em independently} of the dataset $V$. In many settings, however, the objective $f$ depends itself on the entire dataset. In such a setting, we cannot use \Alg as presented above, since we cannot evaluate $f$ on the individual machines without access to the full set $V$.
Fortunately, many such functions have a simple structure which we call \textit{decomposable}. 
More precisely, we call a  submodular function $f$ \textit{decomposable} if it can be written as a sum of submodular functions as follows 
\citep{krause2010budgeted}:
$$f(S) = \frac{1}{|V|}\sum_{i \in V} f_i(S)$$
 In other words, there is separate submodular function associated with every data point $i\in V$. We require that each $f_i$ can be evaluated without access to the full set $V$. Note that the exemplar based clustering application we discussed in Section~\ref{sec:examples} is an instance of this framework, among many others.  
Let us define the evaluation of $f$ restricted to $D\subseteq V$ as follows: 
$$f_D (S) = \frac{1}{|D|}\sum_{i \in D} f_i (S)$$

\noindent In the remaining of this section, we show that assigning each element of the dataset randomly to a machine and running \Alg  will provide a solution that is with high probability  close to the optimum solution. For this, let us  assume that $f_i$'s are bounded, and without loss of generality $0\leq f_i(S)\leq 1$ for $1\leq i\leq |V|, S\subseteq V$. Similar to Section~\ref{sec:geo} we assume that \Alg~performs the partition by assigning elements uniformly at random to the machines. These machines then each greedily optimize $f_{V_i}$. The second stage of \Alg optimizes $f_U$, where $U\subseteq V$ is chosen uniformly at random with size $\lceil n/m\rceil$.

Then, we can show the following result. First, for any fixed $ \epsilon, m, k $, let us define $n_0$ to be the smallest integer such that for  $n\geq n_0$ we have $\ln(n)/n\leq \epsilon^2/(mk)$. 

\begin{theorem}\label{th:decom}
 For $n\geq \max(n_0, m\log(\delta/4m)/\epsilon^2)$, $\epsilon<1/4$, and under the assumptions of Theorem~\ref{thm:alg-sampling}, we have,
with probability at least $1-\delta$, $$f(\Agd{m,\kappa})\geq (1-e^{-\kappa/k})(f(\Ac{k})-2\eps).$$
\end{theorem}
The above result demonstrates why \Alg performs  well on decomposable submodular functions with massive data even when they are evaluated locally on each machine. We will report our experimental results on exemplar-based clustering in the next section.

\begin{figure}[t]
                \centering
                \includegraphics[width=1\textwidth]{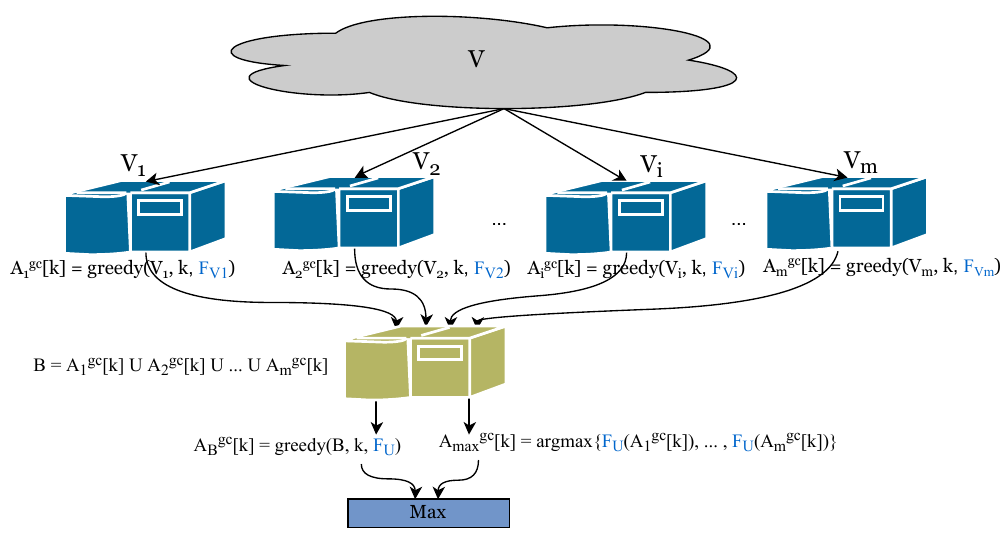}
                \caption{Illustration of our two-round algorithm \Alg for decomposable functions} 
\end{figure}\label{fig:decomp}

\subsection{Performance of {\sc \Alg} on Random Partitions Without Geometric Structure}
Very recently, a constant $(1 - e^{-1})/2$-approximation guarantee was proven for \Alg for the case of random partitioning of the data among the $m$ machines.
\begin{theorem}[\citet{barbosa2015power, mirrokni2015randomized}]
If elements are assigned uniformly at random to the machines, and $\kappa=k$, \Alg gives a $(1-1/e)/2$ approximation guarantee (in the average case) to the optimum centralized solution.
$$\mathbb{E}[f(\Agd{m,k})] \geq \frac{1-1/e}{2} f(\Ac{k}).$$ 
\end{theorem}
These results show that random partitioning of the data is sufficient to guarantee that \Alg provides a constant factor approximation, irrespective of $m$ and $k$, and without the requirement of any geometric structure. On the other hand, if geometric structure is present, the bounds from the previous sections can provide sharper approximation guarantees.

\section{(Non-Monotone) Submodular Functions with General Constraints}\label{sec:constraints}
In this section we show how \Alg can be extended to handle 1) more general constraints, and 2) non-monotone submodular functions. More precisely, we consider the following optimization setting
\begin{eqnarray*}
&&\text{Maximize } f(S)\\
&&\text{Subject to } S\in \C.
\end{eqnarray*}
Here,  we assume that the feasible solutions should be  members of the constraint set $\C \subseteq 2^V$.  The function $f(\cdot)$ is submodular but may not be monotone.  By overloading the notation we denote the set  that achieves the above constrained optimization problem by $\Ac{\C}$. Throughout this section we assume that the constraint set $\C$ is hereditary, meaning that if $A\in \C$ then for any $B\subseteq A$ we also require that $B \in \C$. Cardinality constraints are obviously hereditary, so are all the examples we mention below.

\subsection{Matroid Constraints}

A matroid $\mathcal{M}$ is a pair $(V,\mathcal{I})$ where $V$ is a finite set (called the ground set) and $\mathcal{I} \subseteq 2^V$ is a family of subsets of $V$ (called the independent sets) satisfying the following two properties:
\begin{itemize}
\item
\textit{Heredity property}: $A \subseteq B \subseteq V$ and $B \in \mathcal{I}$ implies that $A \in \mathcal{I}$, i.e. every subset of an independent set is independent.
\item
\textit{Augmentation property}: If $A, B \in \mathcal{I}$ and $|B| > |A|$, there is an element $e \in B \setminus A$ such that $A \cup \{e\} \in \mathcal{I}$.
\end{itemize}

Maximizing a submodular function subject to matroid constraints has found several applications in machine learning and data mining, ranging from 
content aggregation on the web \citep{abbassi2013diversity} to viral marketing \citep{narayanam2012viral} and online advertising \citep{streeter09}. 

One way to approximately maximize a monotone submodular function $f(S)$ subject to the constraint that each $S$ is independent, i.e., $S \in \mathcal{I}$, is to use a generalization of the greedy algorithm.
This algorithm, which starts with an empty set and in each iteration picks the feasible element with maximum benefit until there is no more element $e$ such that $S \cup \{e\} \in \mathcal{I}$, is guaranteed to provide a $\frac{1}{2}$-approximation of the optimal solution \citep{fisher78}. Recently, this bound has been improved to $(1-1/e)$ using the continuous greedy algorithm \citep{calinescu2011maximizing}. 
For non-negative and  non-monotone submodular functions with matroid constraints, the best known result is a 0.325-approximation based on simulated annealing \citep{gharan2011submodular}.

\paragraph{Curvature: }
For a submodular function $f$, the total curvature of $f$ with respect to a set $S$ is defined as:
$$ c = 1 - \min_{j \in V} \frac{f(j |S \setminus j)}{f(j)}.$$
Intuitively, the notion of curvature determines how far away $f$ is from being modular. In other words, it measures how much the marginal gain of an element w.r.t.~set $S$ can decrease as a function of $S$. In general, $c \in [0, 1]$, and for additive (modular) functions, $c = 0$, i.e., the marginal values are independent of $S$. In this case, the greedy algorithm returns the optimal solution to  $\max \{f(S) : S \in \mathcal{I} \}$. In general, the greedy algorithm gives a $\frac{1}{1+c}$-approximation to maximizing a non-decreasing submodular function with curvature $c$ subject to a matroid constraint \citep{conforti1984submodular}.
In case of the uniform matroid $\mathcal{I} = \{S : |S| \leq k \}$, the approximation factor is $(1-e^{-c})/c$.

\paragraph{Intersection of Matroids:} 
A more general case is when we have $p$ matroids $\mathcal{M}_1 = (V, \mathcal{I}_1), \mathcal{M}_2 = (V, \mathcal{I}_2), ..., \mathcal{M}_p = (V, \mathcal{I}_p)$ on the same ground set $V$, and we want to maximize the submodular function $f$ on the intersection of $p$ matroids. That is, $\mathcal{I} = \bigcap_i \mathcal{I}_i$ consists of all subsets of $V$ that are independent in all $p$ matroids. This constraint arises, e.g., when optimizing over rankings (which can be modeled as intersections of two partition matroids). Another recent application considered is finding the influential set of users in viral marketing when multiple products need to be advertised and each user can tolerate only a small number of recommendations \citep{du2013budgeted}. 
For $p$ matroid constraints, the $\frac{1}{p+1}$-approximation provided by the greedy algorithm \citep{fisher78} has been improved to a $(\frac{1}{p} - \varepsilon)$-approximation for $p \geq 2$ by \cite{lee2009submodular}. 
For the non-monotone 
case, a $1/(p+2+1/p+\varepsilon)$-approximation based on local search is also given by \cite{lee2009submodular} .

\paragraph{$p$-systems:}
$p$-independence systems generalize constraints given by the intersection of $p$ matroids. Given an independence family $\indepset$ and a set $V'\subseteq V$, let $S(V')$ denote the set of maximal independent sets of $\indepset$ included in $V'$, i.e., $S(V')=\{A\in \indepset\mid \forall e\in V'\setminus A: A\cup\{e\}\notin \indepset\}$. 
Then we call $(V, \indepset)$ a $p$-system if for all nonempty $V'\subseteq V$ we have $$\max_{A\in S(V')}|A|\leq p\cdot \min_{A\in S(V')}|A|.$$
Similar to $p$ matroid constraints, the greedy algorithm provides a $\frac{1}{p+1}$-approximation guarantee for maximizing a monotone submodular function subject to a $p$-systems constraint \citep{fisher78}.
For the non-monotone case, a $2/( 3(p+2+1/p) )$-approximation can be achieved by combining an algorithm of \cite{gupta2010constrained} with the result for unconstrained submodular maximization of \cite{buchbinder12tight}. 

\subsection{Knapsack Constraints}
In many applications, including  feature and variable selection in probabilistic models \citep{krause05near} and document summarization \citep{lin2011class}, 
elements $e \in V$ have non-uniform costs $c(e) > 0$, and we 
wish to find a collection of elements $S$ that maximize $f$ subject to the constraint that the total cost of elements in $S$ does not exceed a given budget $\mathcal{R}$, i.e.
\[ \max_S f(S) ~~ \text{s.t.} ~ \sum_{v \in S} c(v) \leq \mathcal{R}. \]

\noindent Since the simple greedy algorithm  ignores cost while iteratively adding elements with maximum marginal gains according (see Eq. \ref{eq:argmax}) until $|S| \leq \mathcal{R}$, it can perform arbitrary poorly. 
However, it has been shown that taking the maximum over the solution returned by the greedy algorithm that works according to Eq. \ref{eq:argmax} and the solution returned by the modified greedy algorithm that optimizes the cost-benefit ratio 
\begin{equation*}\label{eq:mgreedy}
v^* = {\arg\,\max}_{
\substack{
e\in V \setminus S\\
c(v) \leq \mathcal{R} - c(S) }}  \frac{ f (S \cup \{e\}) - f (S) }{c(v)},
\end{equation*}
provides a $(1-1/\sqrt{e})$-approximation of the optimal solution \citep{krause05note}. Furthermore, a more computationally expensive algorithm which starts with all feasible solutions of cardinality 3 and augments them using the cost-benefit greedy algorithm to find the set with maximum value of the objective function provides a $(1-1/e)$-approximation \citep{sviridenko04note}.
For maximizing non-monotone submodular functions subject to knapsack constraints,  a ($1/5 - \varepsilon$)-approximation algorithm based on local search was given by \cite{lee2009non}.

\paragraph{Multiple Knapsack Constraints:}
In some applications such as 
procurement auctions \citep{garg2001approximation}, video-on-demand systems and e-commerce \citep{kulik2009maximizing},
we have a $d$-dimensional budget vector $\mathcal{R}$ and a set of element $e \in V$ where each element is associated with a $d$-dimensional cost vector. In this setting, we seek a subset of elements $S \subseteq V$ with a total cost of at most $\mathcal{R}$ that maximizes a non-decreasing submodular function $f$. \cite{kulik2009maximizing} proposed a two-phase algorithm that provides a $(1-1/e - \varepsilon)$-approximation for the problem by first guessing a constant number of elements of highest value, and then taking the value residual problem with respect to the guessed subset.
For the non-monotone case, \cite{lee2009non} provided a $(1/5 - \varepsilon)$-approximation based on local search.

\paragraph{$p$-system and $d$ knapsack constraints:}
A more general type of constraint that has recently found interesting applications in viral marketing \citep{du2013budgeted} can be constructed by combining a $p$-system with $d$ knapsack constraints 
which comprises the intersection of $p$ matroids or $d$ knapsacks as special cases. 
\cite{badanidiyuru2014fast} proposed a modified version of the greedy algorithm that guarantees a $1/(p+2d+1)$-approximation for maximizing a monotone submodular function subject to $p$-system and $d$ knapsack constraints.

Table \ref{appr} summarizes the approximation guarantees for monotone and non-monotone submodular maximization under different constraints.

\begin{table}[t!]
\begin{center}
    \begin{tabular}{|| p{2cm} || p{5.4cm} | p{7.cm}|}
    \hline
    Constraint & \multicolumn{2}{|c|}{Approximation ($\tau$)} \\
& monotone submodular functions & non-monotone submodular functions \\
    \hline \hline
    Cardinality & $1-1/e$ \citep{fisher78} & 0.325 \citep{gharan2011submodular}\\ \hline
    1 matroid & $1-1/e$ \citep{calinescu2011maximizing} & 0.325 \citep{gharan2011submodular} \\ \hline
    $p$ matroid & $1/p - \varepsilon$ \citep{lee2009submodular} & $1/(p+2 +1/p + \varepsilon) $ \citep{lee2009submodular} \\ \hline
    1 knapsack & $1-1/e$ \citep{sviridenko04note} & 1/5 - $\varepsilon$ \citep{lee2009non} \\ \hline
    $d$ knapsack & $1-1/e - \varepsilon$ \cite{kulik2009maximizing} & 1/5 - $\varepsilon$ \citep{lee2009non} \\ \hline 
    $p$-system & 1/($p+1$) \citep{fisher78} & $2/( 3(p+2+1/p) )$ \citep{gupta2010constrained}\\ \hline   
    $p$-system + $d$ knapsack & $1/(p+2d+1)$ \citep{badanidiyuru2014fast} & -- \\
    \hline    
    \end{tabular}
\end{center}
\captionsetup{format=hang}
\caption{
Approximation guarantees ($\tau$) for monotone and non-monotone submodular maximization under different constraints.
}\label{appr}
\end{table}

\subsection{{\sc \Alg} Approximation Guarantee under More General Constraints}
Assume that we have a set of constraints $\C \subseteq 2^V$ that is hereditary. Further assume we have access to a "black box" algorithm $\BBalg$ ~that gives us a constant factor approximation guarantee for maximizing a non-negative (but not necessarily monotone) submodular  function $f$ subject to $\C$, i.e.

\begin{equation}
\BBalg : (f, \zeta) \mapsto A^{\BBalg} \in \zeta ~~ \text{s.t.} ~~ f(A^{\BBalg}[\zeta]) \geq \tau \max_{A \in \zeta} f(A).
\end{equation}

\noindent We can modify \Alg to use any such approximation algorithm as a black box, and provide theoretical guarantees about the solution. 
In order to process a large dataset, it first distributes the ground set over $m$ machines. Then instead of greedily selecting elements, each machine  $i$--in parallel--separately runs the black box algorithm $\BBalg$~on its local data in order to produce a feasible set $A^{\BBalg}_i[\zeta]$ meeting the constraints $\zeta$. We denote by $\Oa{\zeta}$ the set with maximum value among $A^{\BBalg}_i[\zeta]$.
Next, the solutions are merged: $B = \cup_{i=1}^m A^{\BBalg}_i [\zeta]$, and the black box algorithm is applied one more time to set $B$ to produce a solution $\Ob{\zeta}$. 
 Then,
the distributed solution for parameter $m$ and constraints $\zeta$, $A^{\BBalg d}[m,\zeta]$, is the best among $\Oa{\zeta}$ and $\Ob{\zeta}$.
This procedure is given in more detail in Algorithm \ref{alg:gdist_const}. 

\begin{figure}[h]
\vspace{-.5cm}
\begin{algorithm}[H]
\caption{\Alg under General Constraints} 
\label{alg:gdist_const}
\begin{algorithmic}[1]
\REQUIRE{Set $V$, $\# $of partitions $m$, constraints  $\zeta$, submodular function $f$.}
\ENSURE{Set $A^{\BBalg d}[m,\zeta]$.}
\STATE Partition $V$ into $m$ sets $V_1, V_2, \dots, V_m$.
\STATE \looseness -1 In parallel: Run the approximation algorithm \BBalg ~on each set $V_i$ to find a solution $A^{\BBalg}_i[\zeta]$.
\STATE Find $\Oa{\zeta} = \arg\max_A\{F(A) | A\in\{A^{\BBalg}_1[\zeta],\dots,A^{\BBalg}_m[\zeta]\}\}$.
\STATE Merge the resulting sets: $B = \cup_{i=1}^m A^{\BBalg}_i[\zeta].$
\STATE Run the approximation algorithm $\BBalg$~on $B$ to find a solution $\Ob{\zeta}$.
\STATE Return $A^{\BBalg d}[m,\zeta] = \arg \max \{ \Oa{\zeta}, \Ob{\zeta} \}$.
\end{algorithmic}
\end{algorithm}
\hspace{0.05cm}
\end{figure}
The following result generalizes Theorem~\ref{th:ggreedy} for maximizing a submodular function subject to more general constraints.
\begin{theorem}\label{th:ggreedy_const}
Let $f$ be a non-negative submodular function and $\BBalg$~be a black box algorithm that provides a $\tau$-approximation guarantee for submodular maximization subject to a set of hereditary constraints $\zeta$. Then 
$$f(A^{\BBalg d}[m,\zeta])) \geq \frac{\tau}{\min \big( m,\rho([\zeta]) \big)} f(\Ac{\zeta}),$$ 
 where $f(\Ac{\zeta})$ is the optimum centralized solution, and $\rho([\zeta]) = \max_{A \in \zeta} |A|$. 
%
\end{theorem}

\noindent  Specifically, for submodular maximization subject to the matroid constraint $\matroid$, we have $\rho([A \in \mathcal{I}]) = r_{\matroid}$ where $r_{\matroid}$ is the rank of the matroid (i.e., the maximum size of any independent set in the system). For submodular maximization subject to the knapsack constraint $\mathcal{R}$, we can  bound $\rho([c(A) \leq \mathcal{R}])$  by $\lceil\mathcal{R}/ \min_v c(v)\rceil$ (i.e. the capacity of the knapsack divided by the smallest weight of any element).

\paragraph{Performance on Datasets with Geometric Structure.} When the submodular function $f(\cdot)$ and the constraint set $\C$ have more structure, then we can provide much better approximation guarantees.
Assuming the elements of $V$ are embedded in metric space with distance $d: V\times V\rightarrow \mathbb{R}^+$
, we say that  $\C$ is \textit{locally replaceable} with respect to a set $S \subseteq V$ with parameter $\alpha>0$ if 
$$\forall S' \subseteq V \text{ s.t. } |S'| = |S| \text{ and } d_{\infty}(S,S')\leq \alpha \Rightarrow S'\in \C.$$ Here, we define the distance $d_{\infty}$ between two sets $S$ and $S'$ of the same size $k$ as follows. Let $M$ be the set of all possible matchings between $S$ and $S'$, i.e., $$M = \{((e_1, e'_1), \dots, (e_k,e'_k)) \text{ s.t } e_i\in S \text{ and } e'_i\in S' \text{ for } 1\leq i\leq k\}.$$ Then $d_{\infty}(S,S') = \min_M\max_i d(e_i,e'_i)$.
We require locality only with respect to $\optsol$
to ensure that the optimum solution can be well approximated. What the locally replaceable property requires is that as elements of $\optsol$ get replaced by nearby elements, the resulting set is also a feasible solution. Combining this property with $\lambda$-Lipschitzness 
will provide us with the following theorem.

\begin{theorem}\label{thm:BBalg-neighborhoods}
Under the conditions that 1) elements are assigned uniformly at random to $m$ machines, 2) for each $e_i\in \Ac{\zeta}$ we have $\abs{N_\alpha(e_i)}\geq \rho([\zeta]) m \log(\rho([\zeta])/\delta^{1/m})$, 3) $f(\cdot)$ is $\lambda$-Lipschitz, and 4) $\C$ is locally replaceable with respect to $\optsol$ with parameter $\alpha$,
then with probability at least $(1- \delta)$, $$f(A^{\BBalg d}[m,\zeta]))\geq \tau(f(\Ac{\zeta})-\lambda\alpha \rho([\zeta])).$$
\end{theorem}
The above result generalizes Theorem~\ref{thm:alg-neighborhoods} for maximizing non-negative submodular functions subject to different constraints.

\paragraph{Performance Guarantee for Very Large datasets.} Similarly, we can generalize Theorem~\ref{thm:alg-sampling} for maximizing non-negative submodular functions subject to more general constraints.
Suppose that our dataset is a finite sample $V$ drawn i.i.d. from an underlying {\em infinite} set $\mathcal{V}$, according to some (unknown) probability distribution. Let $\optsol$ be an optimal solution in the infinite set, i.e., $\optsol = \arg\max_{S \subseteq \mathcal{V}} f(S)$, such that around each $e_i\in \optsol$, there is a neighborhood of radius at least $\alpha^*$ where the probability density is at least $\beta$ at all points (for some constants $\alpha^*$ and $\beta$). 
Recall that $g:\reals \rightarrow \reals$ is the  growth function where $g(\alpha)$ measures the volume of a ball of radius $\alpha$ centered at a point in the metric space. 
\begin{theorem}\label{thm:BBalg-sampling}
For $\dst n\geq \frac{8\rho([\zeta]) m\log(\rho([\zeta])/\delta^{1/m})}{\beta  g(\frac{\eps}{\lambda \rho([\zeta])})}$, where $\frac{\eps}{\lambda \rho([\zeta])}\leq \alpha^*$, if  \Alg~assigns elements uniformly at random to $m$ processors and under the conditions that $f$ is $\lambda$-Lipschitz, and $\C$ is locally replaceable with respect to $\Ac{\zeta}$ with parameter $\alpha^*$, 
then with probability at least $(1- \delta)$, we have  $$f(A^{\BBalg d}[m,\zeta]))\geq \tau(f(\Ac{\zeta})-\eps).$$
\end{theorem}

\paragraph{Performance Guarantee for Decomposable Functions.}
 For the case of decomposable functions described in Section~\ref{sec:extension}, 
the following generalization of Theorem~\ref{th:decom} holds for maximizing a non-negative submodular function subject to more general constraints. Let us define $n_0$ to be the smallest integer such that for  $n\geq n_0$ we have $\ln(n)/n\leq \epsilon^2/(m \cdot\rho([\zeta]))$. 
\begin{theorem}\label{th:BBalg-decom}
 For $n\geq \max(n_0, m\log(\delta/4m)/\epsilon^2)$, $\epsilon<1/4$, and under the assumptions of Theorem~\ref{thm:BBalg-sampling}, we have,
with probability at least $1-\delta$, $$f(A^{\BBalg d}[m,\zeta]))\geq \tau(f(\Ac{\zeta})-2\eps).$$
\end{theorem}

\vspace{-.1cm}

\section{Experiments}\label{sec:experiments}
In our experimental evaluation we wish to address the following questions:
1) how well does \Alg perform compared to the centralized solution, 2) how good is the performance of \Alg  when using decomposable objective functions (see Section~\ref{sec:extension}), and finally 3) how well does \Alg scale in the context of massive datasets. To this end, we run  \Alg  on three scenarios: exemplar based clustering, active set selection in GPs and finding the maximum cuts in graphs. 

We compare the performance of our \Alg method to the following naive approaches: 
\begin{itemize}
\item  {\em random/random}: in the first round each machine simply outputs $k$ randomly chosen elements from its local data points and in the second round  $k$ out of the merged $mk$ elements, are again randomly chosen as the final output. 
\item {\em random/greedy}: each machine  outputs $k$ randomly chosen elements from its local data points, then the standard greedy algorithm is run over $mk$ elements to find a solution of size $k$. 
\item {\em greedy/merge}: in the first round $k/m$ elements are chosen greedily from each machine and in the second round  they are merged to output a solution of size $k$. \item {\em greedy/max}: in the first round each machine  greedily finds a solution of size $k$ and in the second round the solution with the maximum value is reported. 
\end{itemize}
For \Alg, we let each of the $m$ machines select a set of size $\alpha k$, and select a final solution of size $k$  among the union of the $m$ solutions (i.e., among $\alpha k m$ elements). We present the performance of \Alg for different parameters $\alpha>0$. For datasets where we are able to find the centralized solution, we report the ratio of $f(A_{\text{dist}}[k])/f(\Agc{k})$, where $A_{\text{dist}}[k]$ is the distributed solution (in particular $\Agd{m,\alpha k,k}=A_{\text{dist}}[k]$ for \Alg).

\subsection{Exemplar Based Clustering} Our exemplar based clustering experiment involves \Alg applied to the clustering utility $f(S)$ (see Sec.~\ref{sec:examples}) with $d(x,x')= \| x-x' \|^2$.
We performed our experiments on a set of 10,000 \textit{Tiny Images} \citep{torralba200880}. 
Each 32 by 32 RGB pixel image was represented by a 3,072 dimensional vector. 
We subtracted from each vector the mean value, 
 normalized it to unit norm, and used the origin as the auxiliary exemplar.
Fig.~\ref{fig:tiny} compares the performance of our approach to the benchmarks
with the number of exemplars set to $k = 50$, and varying  number of partitions $m$. 
It can be seen that \Alg significantly outperforms the benchmarks  
and provides a solution that is very close to the centralized one. Interestingly, even for very small $\alpha=\kappa/k<1$, \Alg performs very well.
Since the exemplar based clustering utility function is decomposable, we repeated the experiment for the more realistic case where the function evaluation in each machine was restricted to the local elements of the dataset in that particular machine (rather than the entire dataset). Fig \ref{fig:tiny_local} shows similar qualitative behavior for decomposable objective functions.

\begin{figure}[t]
        \centering
        \begin{subfigure}[b]{0.5\textwidth}
                \centering
                \includegraphics[width=.9\textwidth, height=.9\textwidth]{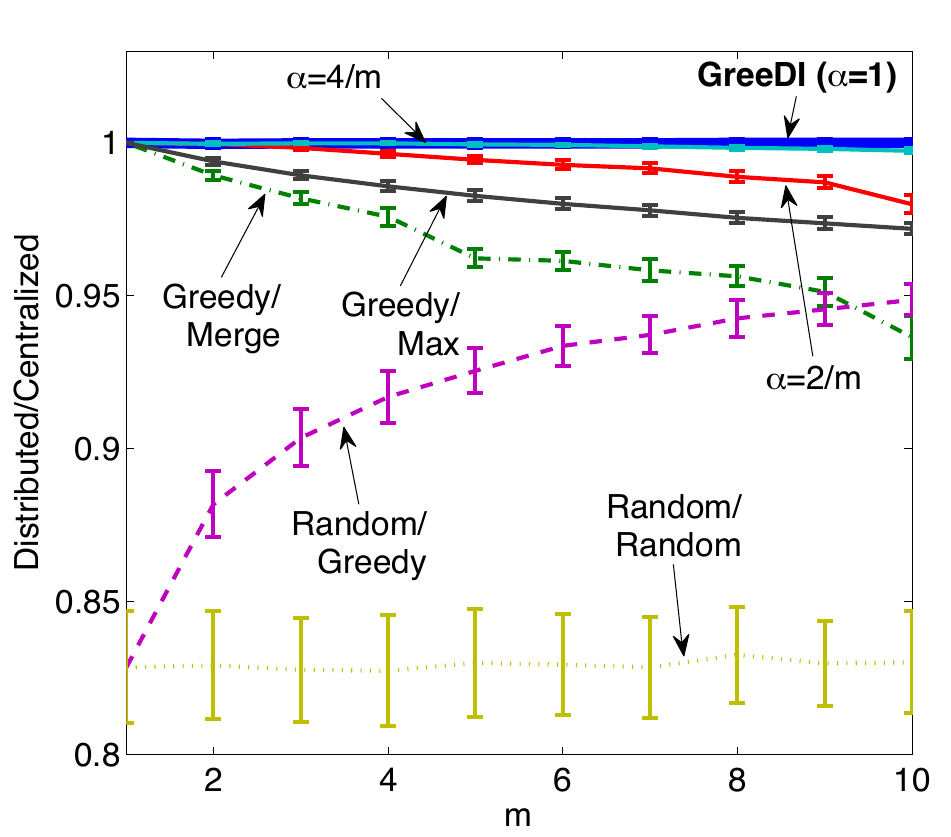}
                \vspace{-0.2cm}
                \caption{Global objective function}
                \label{fig:tiny}
        \end{subfigure}%
        \hspace{-.2cm}
        \begin{subfigure}[b]{0.5\textwidth}
                \centering
                \includegraphics[width=.9\textwidth, height=.88\textwidth]{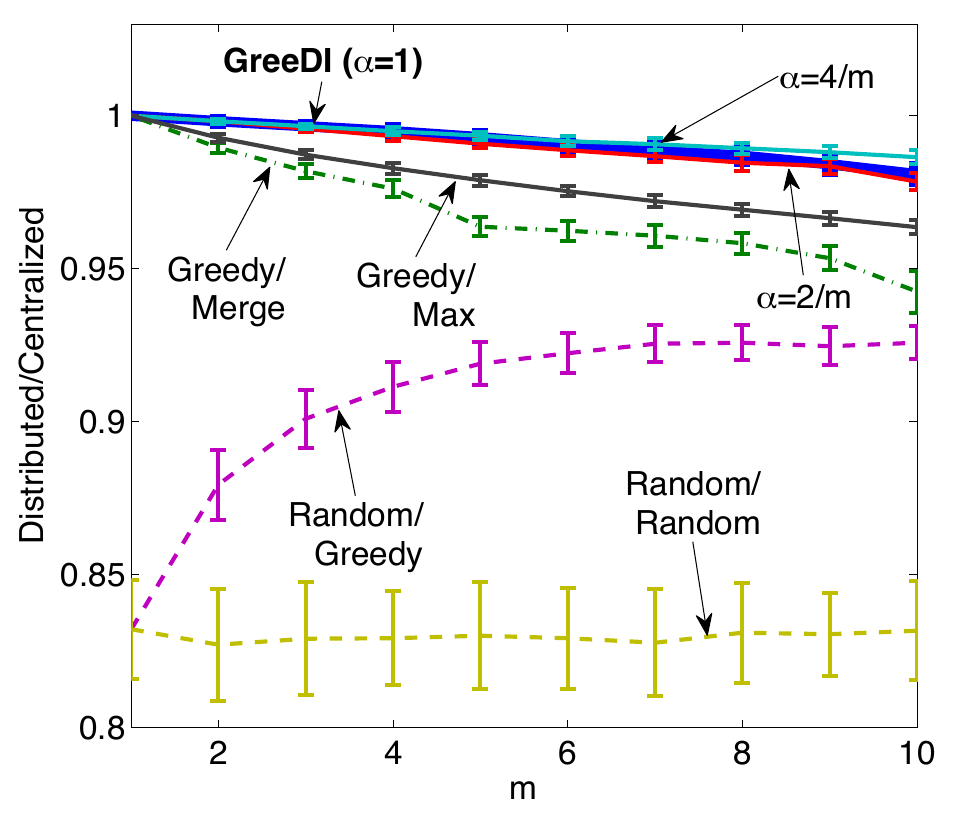}
                \vspace{-0.2cm}
                \caption{Local objective function}
                \label{fig:tiny_local}
        \end{subfigure}
        \hspace{-.2cm}
                \begin{subfigure}[b]{0.5\textwidth}
                \centering
                \includegraphics[width=.9\textwidth, height=.89\textwidth]{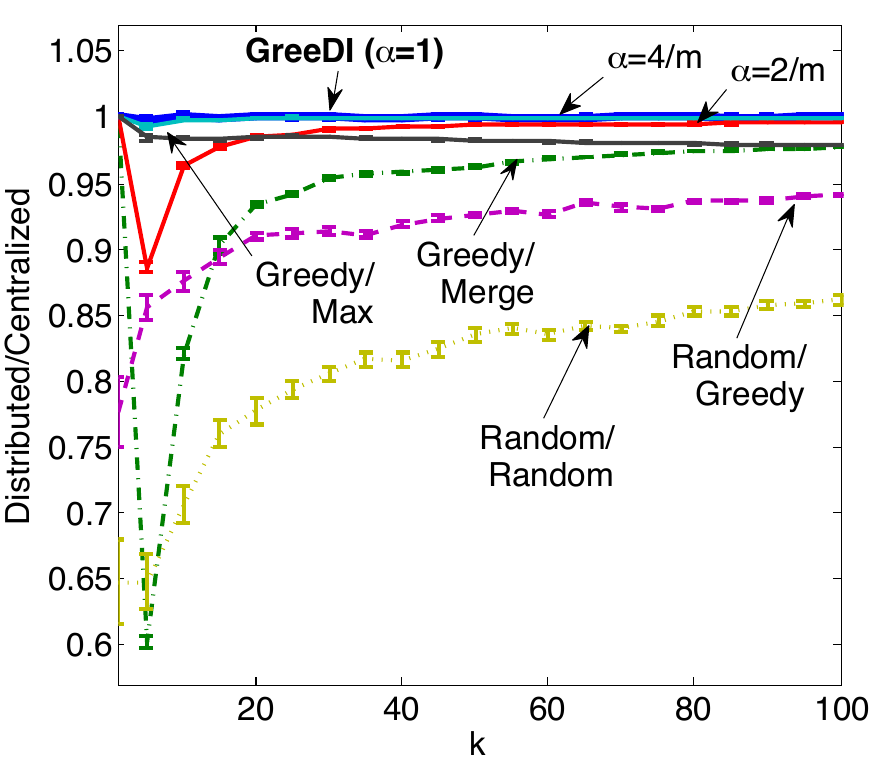}
                \vspace{-0.2cm}
                \caption{Global objective function}
                \label{fig:tinyk}
        \end{subfigure}%
        \hspace{-.2cm}
        \begin{subfigure}[b]{0.5\textwidth}
                \centering
                \includegraphics[width=.9\textwidth, height=.88\textwidth]{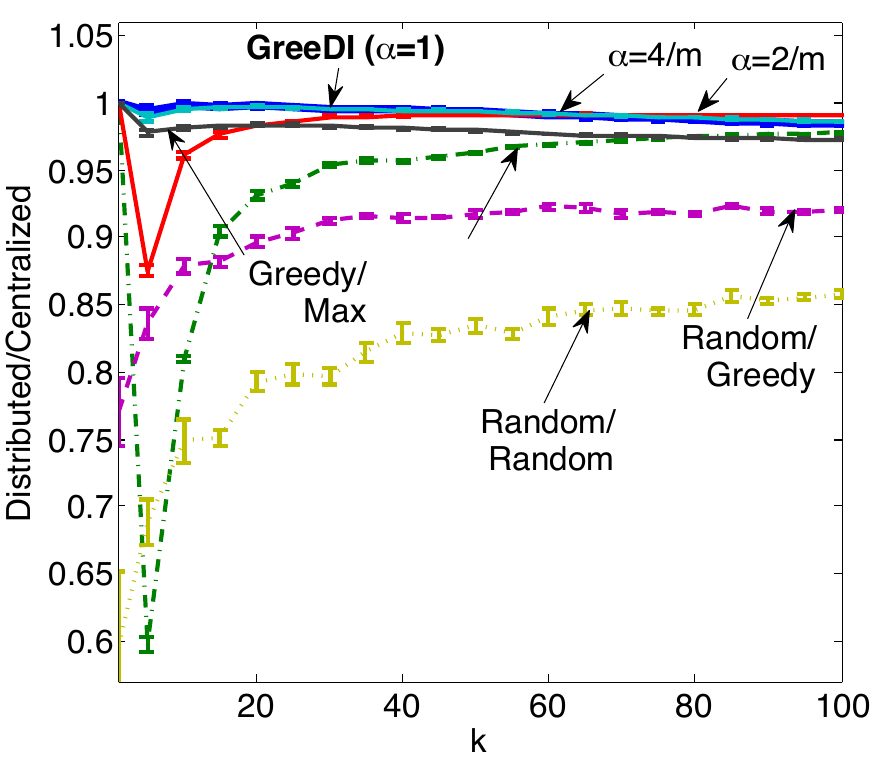}
                \vspace{-0.2cm}
                \caption{Local objective function}
                \label{fig:tinyk_local}
        \end{subfigure}
        \captionsetup{format=hang}
        \caption{
 		Performance of \Alg compared to the other benchmarks. a) and b) show the mean and standard deviation of the ratio of distributed vs.~centralized solution for global and local objective functions with budget $k = 50$ and varying the number $m$ of partitions.	 		
c) and d) show the same ratio for global and local objective functions for $m = 5$ partitions and varying budget $k$, for a set of 10,000 {\em Tiny Images}. 
  		}\label{fig:exp}
\end{figure}

\noindent \textbf{\textit{Large scale experiments with Hadoop.}} 
As our first large scale experiment, we applied \Alg to the whole dataset of 80,000,000 \textit{Tiny Images} \citep{torralba200880} in order to select a set of 64 exemplars. Our experimental infrastructure was a cluster of 10 quad-core machines running Hadoop with the number of reducers set to $m = 8000$. Hereby, each machine carried out a set of reduce tasks in sequence. We first partitioned the images uniformly at  random to  reducers. Each reducer separately performed the lazy greedy algorithm on its own set of 10,000 images ($\approx$123MB) to extract 64 images with the highest marginal gains w.r.t. the  local elements of the dataset in that particular partition. 
We then merged the results and performed another round of lazy greedy selection on the merged results to extract the final 64 exemplars. Function evaluation in the second stage was performed w.r.t a  randomly selected subset of 10,000 images from the entire dataset. The maximum running time per reduce task was 2.5 hours.
As Fig. \ref{fig:tiny_80} shows, \Alg highly outperforms the other distributed benchmarks and can  scale well to very large datasets. Fig. \ref{fig:exemplars} shows a set of cluster exemplars discovered by \Alg where  Fig. \ref{fig:NNeighbors26} and Fig. \ref{fig:NNeighbors63} show 100 nearest images to exemplars 26 and 63 (shown with red borders) in Fig. \ref{fig:exemplars}.

\begin{figure}[t]
        \centering
		\begin{subfigure}[b]{0.5\textwidth}
                \centering
                \includegraphics[width=.9\textwidth, height=.9\textwidth]{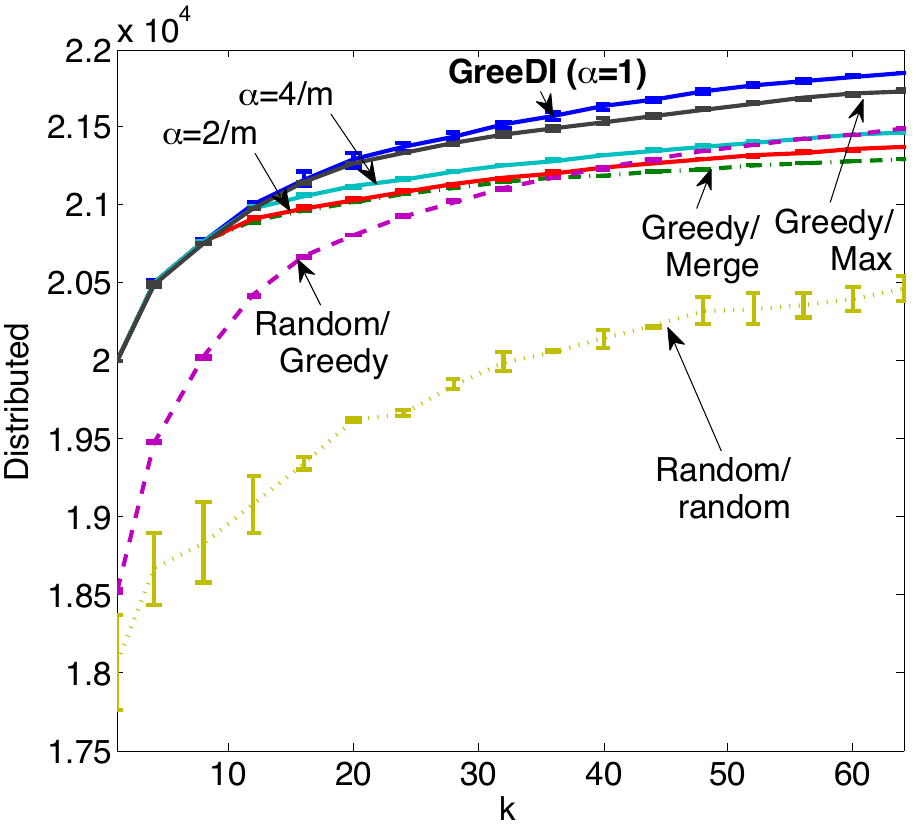}
                \vspace{-0.2cm}
                \caption{Tiny Images 80M}
                \label{fig:tiny_80}
        \end{subfigure}      
        \hspace{4pt}
        \begin{subfigure}[b]{0.45\textwidth}
                \centering
                \includegraphics[width=\textwidth]{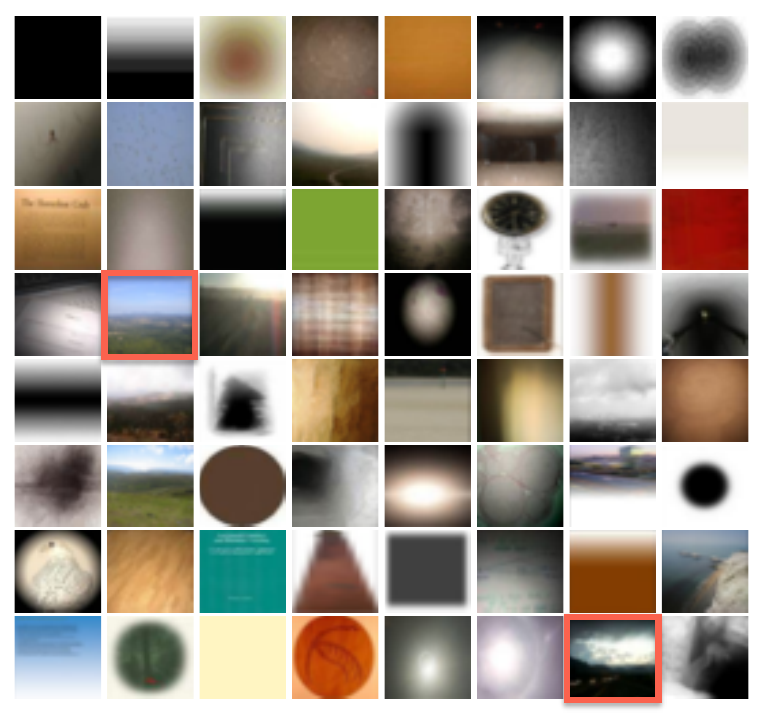}
                \vspace{-0.5cm}
                \caption{}
                \label{fig:exemplars}
        \end{subfigure}     

        \begin{subfigure}[b]{0.435\textwidth}
                \hspace{20pt}
                \includegraphics[width=\textwidth]{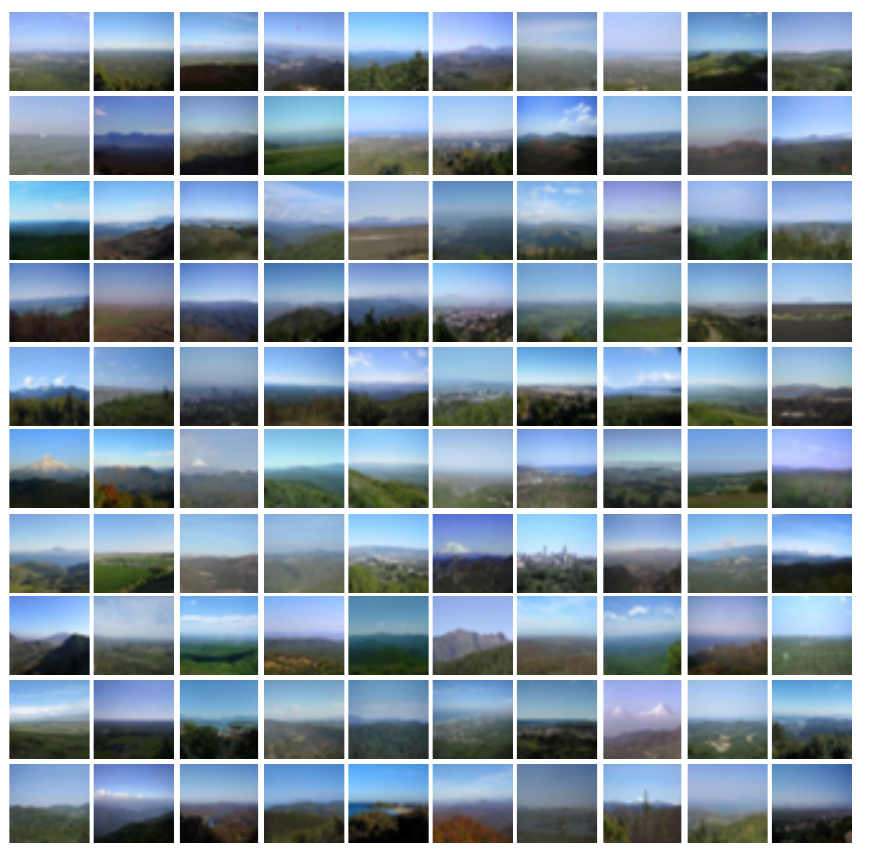}
                \caption{}
                \label{fig:NNeighbors26}
        \end{subfigure} 
        \hspace{30pt}
        \begin{subfigure}[b]{0.46\textwidth}
                \centering
                \includegraphics[width=\textwidth]{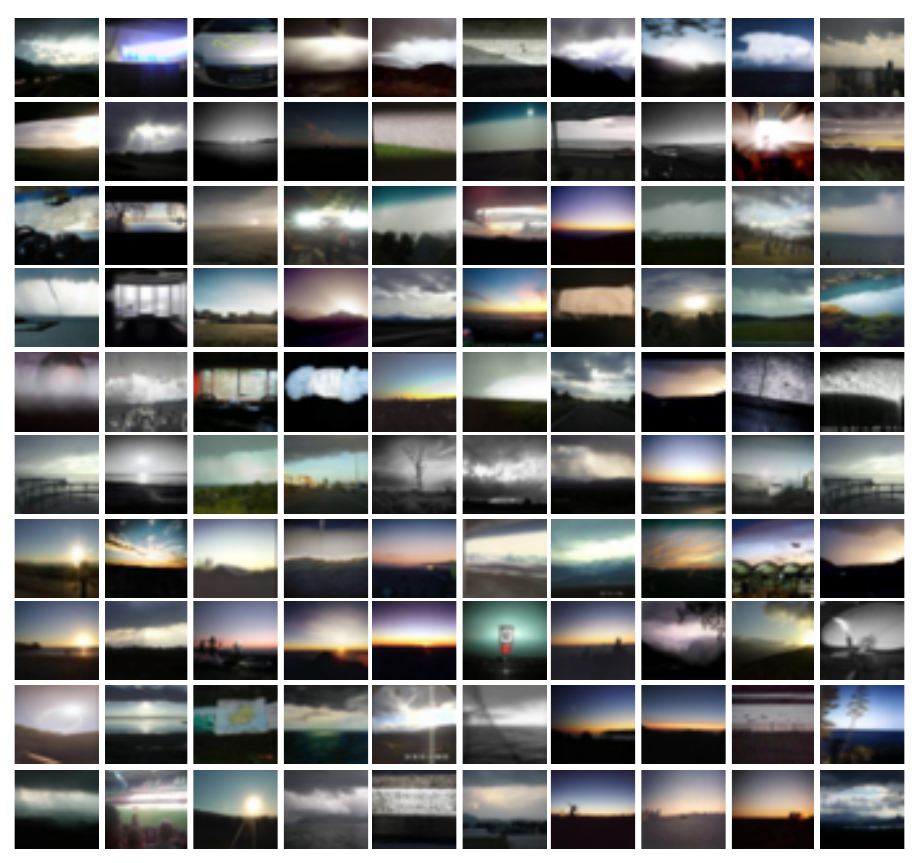}
                \caption{}
                \label{fig:NNeighbors63}
        \end{subfigure}
        \captionsetup{format=hang}  
        \caption{
 		Performance of \Alg compared to the other benchmarks. a) shows the distributed solution with $m = 8000$ and varying $k$ for local objective functions on the whole dataset of 80,000,000 {\em Tiny Images}. b) shows a set of cluster exemplars discovered by \Alg, and each column in c) shows 100 images nearest to exemplars 26 and d) shows 100 images nearest to exemplars 63 in b). 		
 		}\label{fig:exp_tiny}
    \vspace{-.4cm}
\end{figure}        
        
\subsection{Active Set Selection} Our active set selection experiment involves \Alg applied to the information gain $f(S)$ (see Sec. \ref{sec:examples}) with Gaussian kernel, $h = 0.75$ and $\sigma = 1$. We used the \textit{Parkinsons Telemonitoring} dataset \citep{tsanas2010enhanced} consisting of 5,875 bio-medical voice measurements with 22 attributes from people with early-stage Parkinson's disease. 
We normalized the vectors to zero mean and unit norm.
Fig. \ref{fig:parkinson_k} compares the performance \Alg to the benchmarks with fixed $k = 50$ and varying  number  of partitions $m$. Similarly, Fig \ref{fig:parkinson_m} shows the results for fixed $m = 10$ and varying $k$. We find that \Alg significantly outperforms the benchmarks.

\begin{figure}[t]
        \begin{subfigure}[b]{0.5\textwidth}
                \centering
                \includegraphics[width=.9\textwidth, height=.89\textwidth]{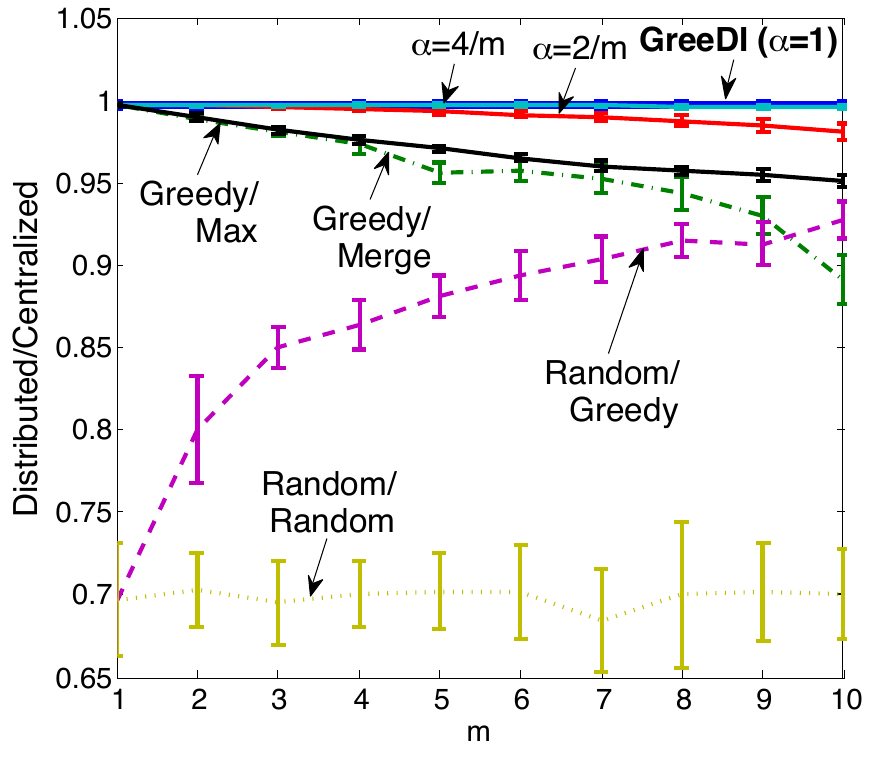}
                \vspace{-0.2cm}
                \caption{Parkinsons Telemonitoring}
                \label{fig:parkinson_m}
        \end{subfigure}%
        \begin{subfigure}[b]{0.5\textwidth}
                \centering
                \includegraphics[width=.9\textwidth, height=.89\textwidth]{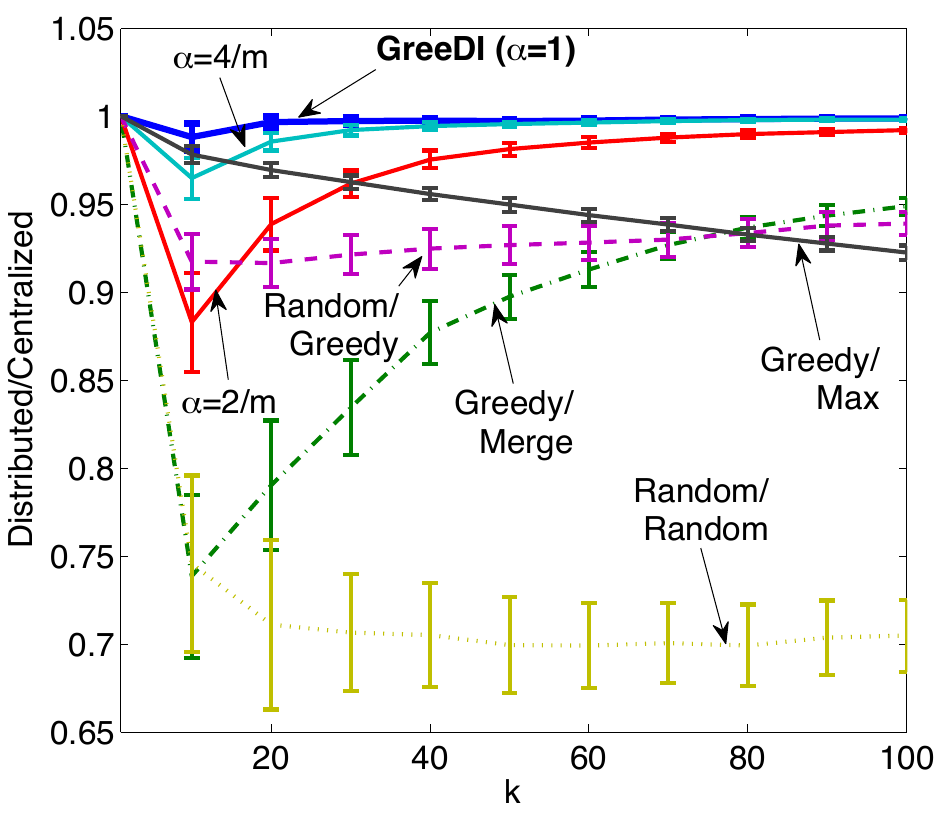}
                \vspace{-0.2cm}
                \caption{Parkinsons Telemonitoring}
                \label{fig:parkinson_k}
        \end{subfigure}
        \captionsetup{format=hang}
        \caption{
 		Performance of \Alg compared to the other benchmarks. a) shows the ratio of distributed vs.~centralized solution with $k = 50$ and varying $m$ for {\em Parkinsons Telemonitoring}. b) shows the same ratio with $m = 10$ and varying $k$ on the same dataset. 
 		}\label{fig:exp_park}
    \vspace{-.4cm}
\end{figure}

\noindent \textbf{\textit{Large scale experiments with Hadoop.}} 
\looseness -1 Our second large scale experiment
consists of 45,811,883 user visits from the Featured Tab of the Today Module on Yahoo! Front Page \citep{website:webscope}. 
For each visit, both the user and each of the candidate articles are associated with a feature vector of dimension 6. Here, we used the normalized user features. Our experimental setup was a cluster of 8 quad-core machines running Spark with the number of reducers set to $m = 32$. 
Each reducer performed the lazy greedy algorithm on its own set of $\approx$1,431,621 vectors ($\approx$34MB) in order to extract 256 elements with the highest marginal gains w.r.t the local elements of the dataset in that particular partition. We then merged the results and performed another round of lazy greedy selection on the merged results to extract the final active set of size 256. 
The maximum running time per reduce task was 12 minutes for selecting 128 elements and 48 minutes for selecting 256 elements. 
Fig. \ref{fig:webscope} shows the performance of \Alg compared to the benchmarks. We note again that \Alg significantly outperforms the other distributed benchmarks and can  scale well to very large datasets.

\begin{figure}[t]
                \centering
                \includegraphics[width=.5\textwidth, height=.5\textwidth]{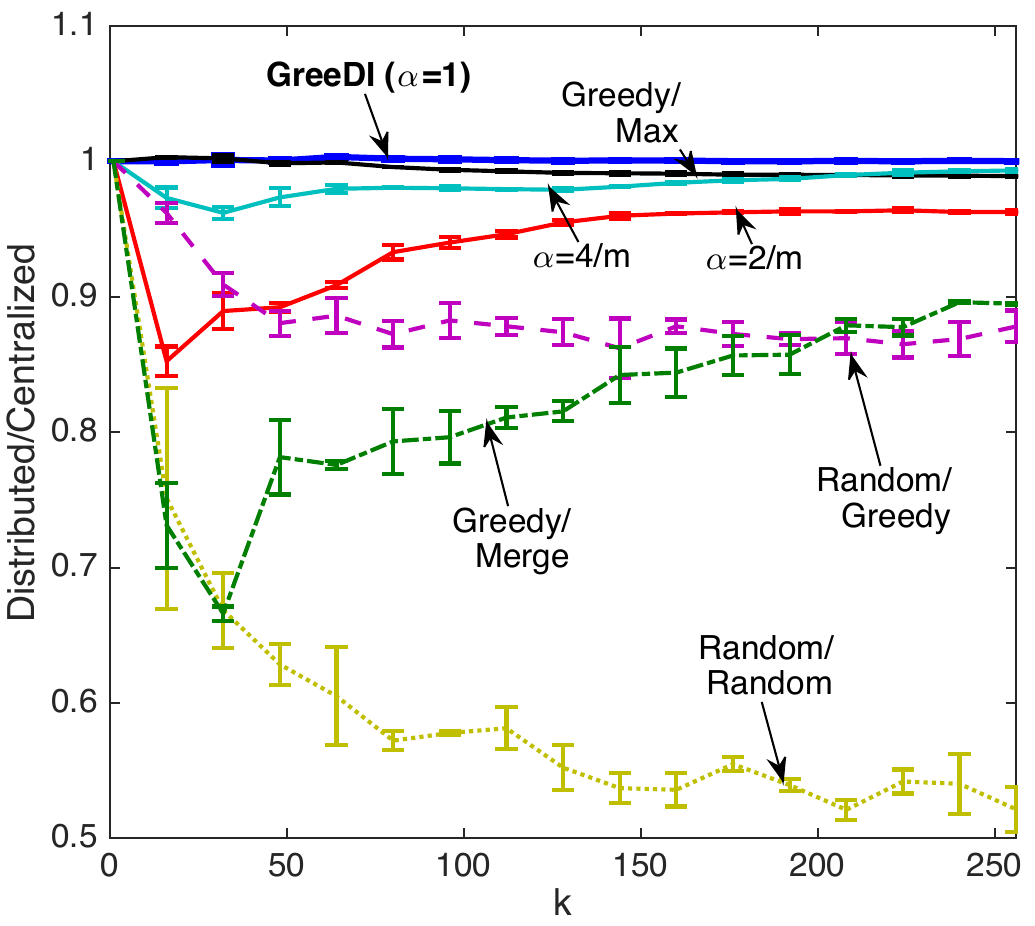}
                \vspace{-0.2cm}
        \captionsetup{format=hang}
        \caption{
 		Performance of \Alg with $m = 32$ and varying budget $k$ compared to the other benchmarks on {\em Yahoo!~Webscope data}.}\label{fig:webscope}
    \vspace{-.4cm}
\end{figure}

\noindent \textbf{\textit{Performance Comparison.}}
Fig. \ref{fig:speed} shows the speedup of \Alg compared to the centralized greedy benchmark for different values of $k$ and varying number of partitions $m$.
As Fig. \ref{fig:speed_m} shows, for small values of $m$, the speedup is almost linear in the number of machines. However, for large values of $m$ the running time of the second stage of \Alg increases and ultimately dominates the whole running time. Hence, we do not observe a linear speedup anymore. This effect can be observed in Fig. \ref{fig:speed_k}. For larger values of $k$, the speedup is higher on fewer machines, but decreases more quickly by increasing $m$, as the second stage takes longer to complete.

\begin{figure}[t]
        \begin{subfigure}[b]{0.45\textwidth}
                \centering
                \includegraphics[width=1\textwidth, height=1\textwidth]{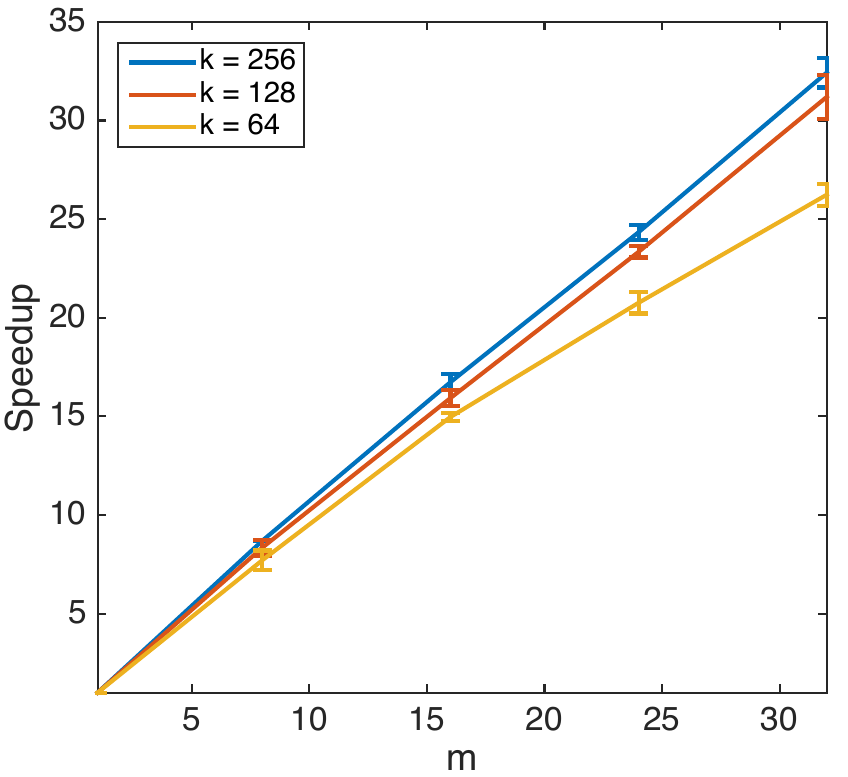}
                \vspace{-0.2cm}
                \caption{Yahoo! front page}
                \label{fig:speed_m}
        \end{subfigure}%
        \hspace{.8cm}
        \begin{subfigure}[b]{0.45\textwidth}
                \centering
                \includegraphics[width=.995\textwidth, height=.995\textwidth]{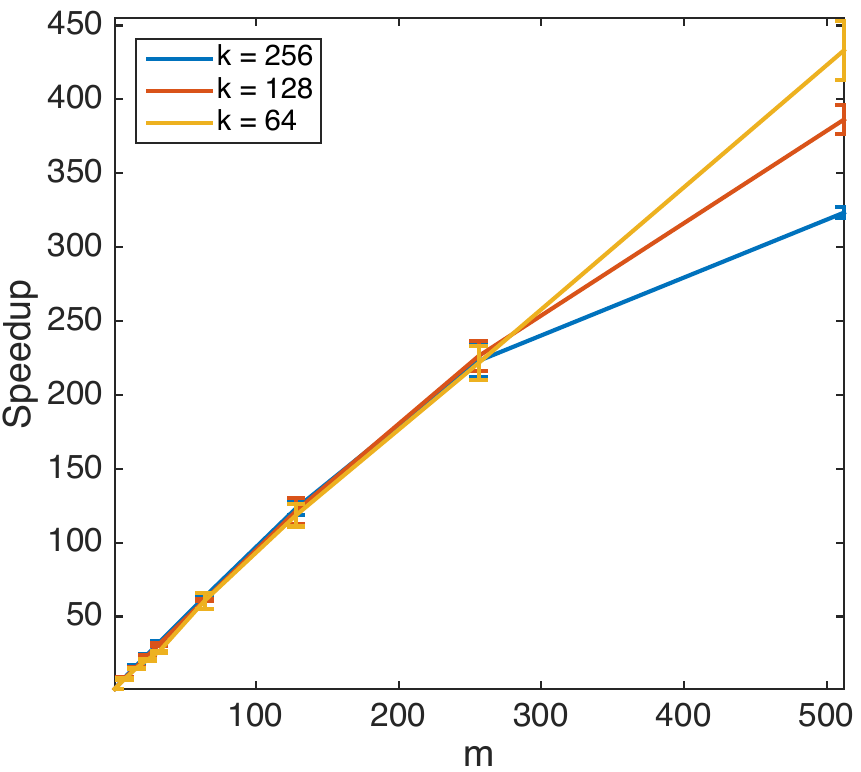}
                \vspace{-0.2cm}
                \caption{Yahoo! front page}
                \label{fig:speed_k}
        \end{subfigure}
        \captionsetup{format=hang}
        \caption{
 		Running time of \Alg compared to the centralized greedy algorithm. a) shows the ratio of centralized vs.~distributed solution with $k = 64, 128, 256$ and up to $m=32$ machines for {\em Yahoo Webscope} data. b) shows the same ratio with $k = 64, 128, 256$ and up to $m=512$ machines on the same dataset. Both experiments are performed on a cluster of 8 quad core machines.
 		}\label{fig:speed}
    \vspace{-.4cm}
\end{figure}

\subsection{Non-Monotone Submodular Function (Finding Maximum Cuts)}
\looseness -1 We also applied \Alg to the problem of finding maximum cuts in graphs. In our setting we used a \textit{Facebook-like social network} \citep{opsahl2009clustering}.
This dataset includes the users that have sent or received at least one message in an online student community at University of California, Irvine and consists of 1,899 users and 20,296 directed ties. Fig. \ref{fig:facebook_m} and \ref{fig:facebook_k} show the performance of \Alg applied to the cut function on graphs. We evaluated the objective function locally on each partition. Thus, the links between the  partitions are disconnected. Since the problem of finding the maximum cut in a graph is non-monotone submodular, we applied the RandomGreedy algorithm proposed by \cite{buchbinder2014submodular} to find the near optimal solution in each partition.

Although the cut function does not decompose additively over individual data points, perhaps surprisingly, \Alg still performs very well, and significantly outperforms the benchmarks. This suggests that our approach is quite robust, and may be more generally applicable.

\begin{figure}[t]
        \begin{subfigure}[b]{0.5\textwidth}
                \centering
                \includegraphics[width=.9\textwidth, height=.89\textwidth]{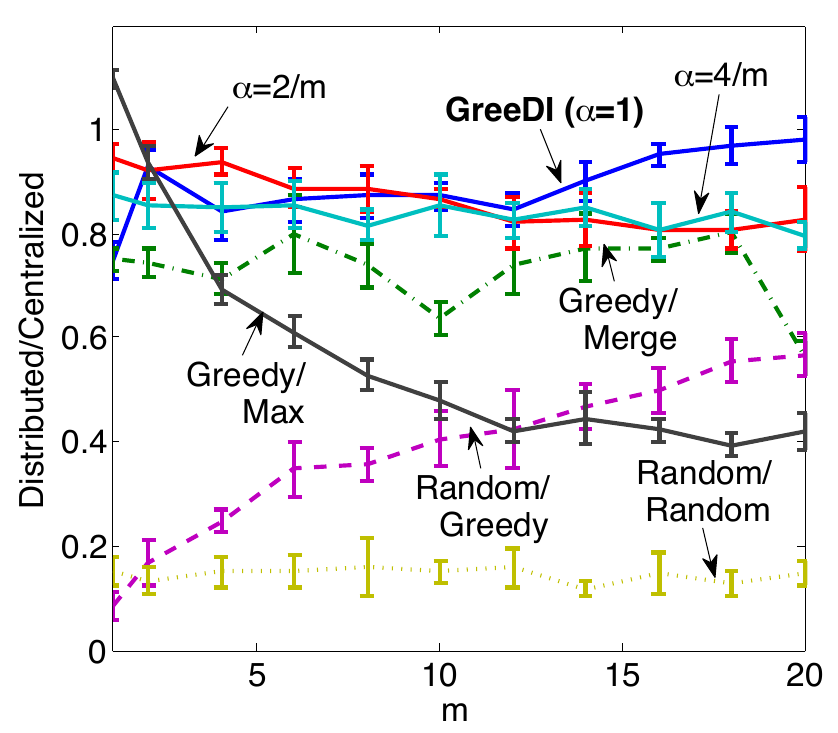}
                \vspace{-0.2cm}
                \caption{Facebook-like social network}
                \label{fig:facebook_m}
        \end{subfigure}%
        \begin{subfigure}[b]{0.5\textwidth}
                \centering
                \includegraphics[width=.9\textwidth, height=.89\textwidth]{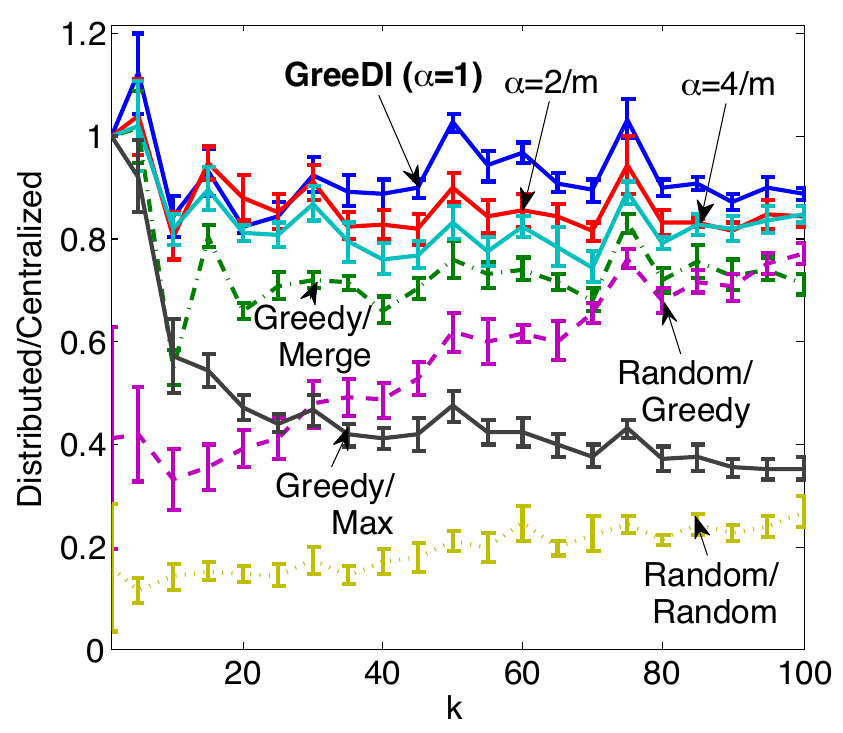}
                \vspace{-0.2cm}
                \caption{Facebook-like social network}
                \label{fig:facebook_k}
        \end{subfigure}
        \captionsetup{format=hang}
\caption{
        Performance of \Alg compared to the other benchmarks. a) shows the mean and standard deviation of the ratio of distributed to centralized solution for budget $k = 20$ with varying number of machines $m$ and b) shows the same ratio for varying budget $k$ with $m=10$ on {\em Facebook-like social network}. 
        }\label{fig:exp_facebook}
    \vspace{-.4cm}
\end{figure}

\subsection{Comparision with {Greedy Scaling}.}\label{sec:comparision}
\cite{kumar2013fast} recently proposed an alternative approach--\textsc{GreedyScaling}--for parallel maximization of submodular functions. \textsc{GreedyScaling} is a randomized algorithm that carries out a number (typically less than $k$) rounds of MapReduce computations. 
We applied \Alg  to the submodular coverage problem in which  given a collection $V$ of sets, we would like to pick at most $k$ sets from $V$ in order to maximize the size of their union.  We compared the performance of our \Alg algorithm to the reported performance of \textsc{GreedyScaling} on the same datasets, namely \textit{Accidents} \citep{geurts2003profiling} and \textit{Kosarak} \citep{website:kosarak}. 
As Fig \ref{fig:ispaa_accident_dc} and \ref{fig:ispaa_kosarak_dc} shows, \Alg  outperforms \textsc{GreedyScaling}  on the {\em Accidents} dataset and its performance is comparable to that of \textsc{GreedyScaling}  in the {\em Kosarak} dataset. 

\begin{figure}[t]
        \begin{subfigure}[b]{0.5\textwidth}
                \centering
                \includegraphics[width=.9\textwidth, height=.89\textwidth]{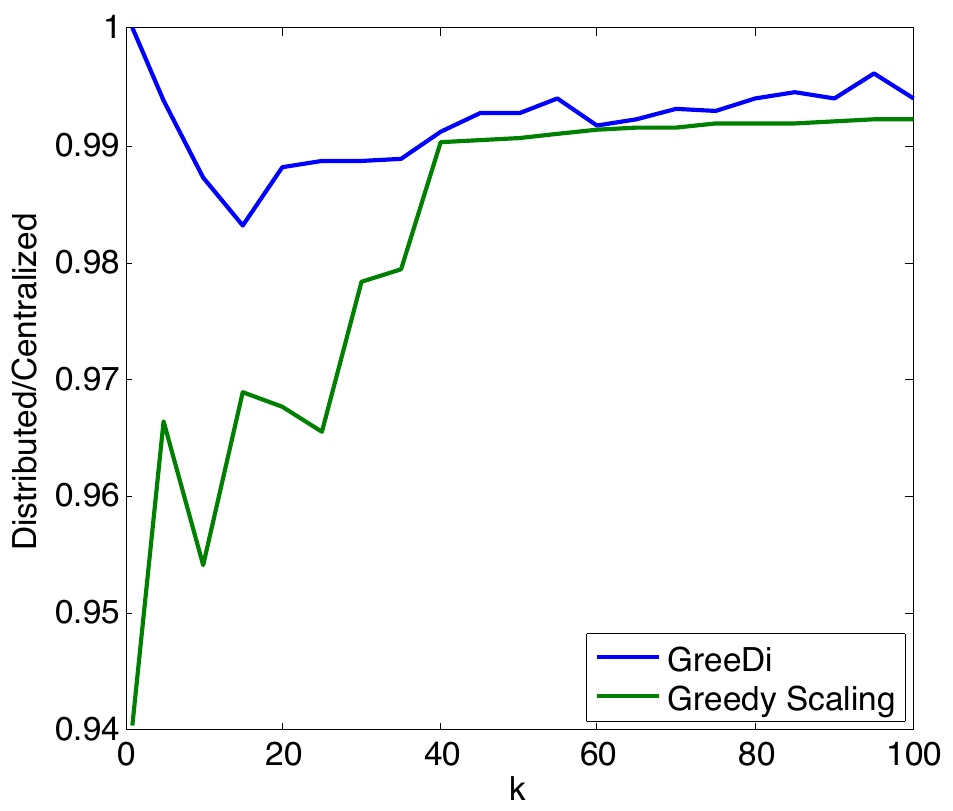}
                \vspace{-0.2cm}
                \caption{Accidents}
                \label{fig:ispaa_accident_dc}
        \end{subfigure}            
          \hspace{-.5cm}
        \begin{subfigure}[b]{0.5\textwidth}
                \centering                
                \includegraphics[width=.93\textwidth, height=.89\textwidth]{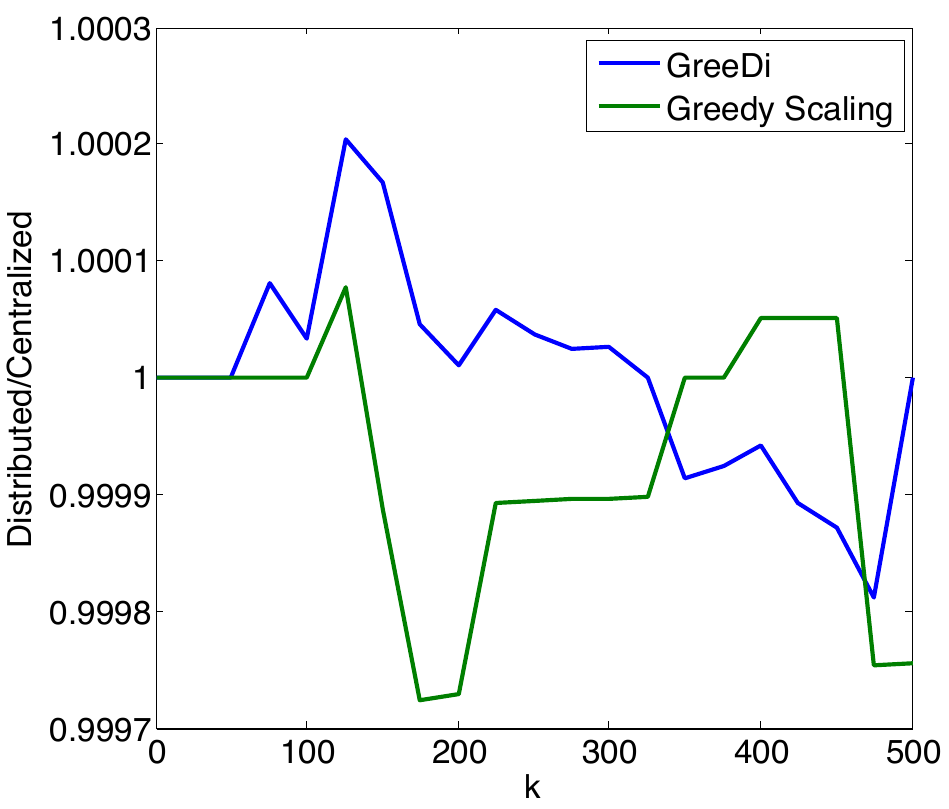}
                \vspace{-0.2cm}
                \caption{Kosarak}
                \label{fig:ispaa_kosarak_dc}
        \end{subfigure}
        \captionsetup{format=hang}
        \caption{
        Performance of \Alg compared to the GreedyScaling algorithm of \cite{kumar2013fast} (as reported in their paper). a) shows the ratio of distributed to centralized solution on \textit{Accidents} dataset with 340,183 elements and b) shows the same ratio for \textit{Kosarak} dataset with 990,002 elements. The results are reported for varying budget $k$ and varying number of machines $m=n/\mu$ where $\mu=O(kn^\delta \log n)$ and $n$ is the size of the dataset. The results are reported for $\delta = 1/2$. Note that the results presented by \cite{kumar2013fast} indicate that GreedyScaling generally requires a substantially larger number of MapReduce rounds compared to \Alg.
        }\label{fig:exp_compare}
    \vspace{-.4cm}
\end{figure}

\vspace{-.2cm}
\section{Conclusion}\label{sec:conclusion}
 We have developed an efficient distributed protocol \Alg,  for constrained submodular function maximization. We have theoretically  analyzed the performance of our method and showed that under certain natural conditions it performs very close to the centralized (albeit impractical in massive datasets) solution. We have also demonstrated the effectiveness  of our approach through extensive experiments, including active set selection in GPs on a dataset of 45 million examples, and exemplar based summarization of a collection of 80 million images using Hadoop. We believe our results provide an important step towards solving submodular optimization problems in very large scale, real applications. 


\vspace{1cm}
\acks{This research was supported by SNF 200021-137971, DARPA MSEE FA8650-11-1-7156, ERC StG 307036, a Microsoft Faculty Fellowship, an ETH Fellowship, and a Scottish Informatics and Computer Science Alliance.}



\appendix

\section{Proofs}\label{sec:analysis}

This section presents the complete proofs of theorems presented in the article.

\subsection{Proof of Theorem~\ref{th:gap}}
\noindent
$\Rightarrow$ direction:\\
The proof easily follows from the following lemmas.

\begin{lemma}
$\dst \max_{i} f(\Aic[k])\geq \frac{1}{m}f(\Ac{k})$.
\label{lemm_m}
\end{lemma}
\begin{proof}
Let $B_i$ be the elements in $V_i$ that are contained in the optimal solution, $ B_i = \Ac{k} \cap V_i $.  
Then we have:
$$ f(\Ac{k}) = f(B_1 \cup \ldots \cup B_m) = f(B_1) + f(B_2 | B_1) + \ldots + f(B_m | B_{m-1} , \ldots , B_1).$$
Using submodularity of $f$, for each $i \in \{1 \ldots m\}$, we have $$f(B_i | B_{i-1} \ldots B_1) \leq f(B_i),$$ and thus,
$$f(\Ac{k}) \leq f(B_1) + \ldots + f(B_m).$$
Since, $ f(\Aic[k]) \geq f(B_i) $, we have $$ f(\Ac{k}) \leq f(A^c_1[k]) + \ldots + f(A^c_m[k]).$$
Therefore, $$f(\Ac{k}) \leq m \; \max_{i} f(\Aic[k]). $$ 
\end{proof}
\begin{lemma}
$\dst \max_{i} f(\Aic[k])\geq \frac{1}{k}f(\Ac{k})$.\label{lemm_k}
\end{lemma}
\begin{proof}
Let $ f(\Ac{k}) = f(\{u_1, \dots u_k \})$. Using submodularity of $f$, we have  
$$ f(\Ac{k}) \leq \sum_{i=1}^{k} f(u_i).$$ Thus, $f(\Ac{k}) \leq k f(u^*) $ where $ u^* = {\arg \max}_i f(u_i)$. 
Suppose that the element with highest marginal gain (i.e., $u^*$) is in $V_j$. 
Then the maximum value of $f$ on $V_j$ would be greater or equal to the marginal gain of $u^*$, i.e., $f(A^c_j[k]) \geq f(u^*)$ and since $f(\max_i f(\Aic[k])) \geq f(A^c_j[k])$, we can conclude that $$f(\max_i f(\Aic[k])) \geq f(u^*) \geq \frac{1}{k} f(\Ac{k}).$$
\end{proof}
Since $f(\Ad{m, k})\geq \max_i f(\Aic[k])$;
from Lemma \ref{lemm_m} and \ref{lemm_k} we have $$f(\Ad{m, k}) \geq \frac{1}{\min(m,k)} f(\Ac{k}).$$

\noindent
$\Leftarrow$ direction:\\
 Let us consider a set of unbiased and  independent Bernoulli random variables $X_{i,j}$ for $i \in \{1, \dots, m\}$ and $j \in \{1, \dots, k\}$, i.e., $\Pr(X_{i,j}=1)=\Pr(X_{i,j}=0)=1/2$ and $(X_{i,j}  \perp X_{i',j'})$ if $i \neq i'$ or $j \neq j'$. Let us also define $Y_i = (X_{i,1}, \dots, X_{i,k})$ for $i\in\{1, \dots, m\}$. Now assume that $ V_i = \{ X_{i,1}, \ldots, X_{i,k}, Y_i\}, V = \bigcup_{i=1}^m V_i $ and $ f(S) = H(S)$, where $H$ is the entropy of the subset $S$ of random variables. Note that $H$ is a monotone submodular function.  It is easy to see that $\Aic[k] = \{ X_{i,1}, \ldots, X_{i,k}\}$ or $\Aic[k] = Y_i$ as in both cases $H(\Aic[k]) = k$. If we assume $\Aic[k] = \{ X_{i,1}, \ldots, X_{i,k}\}$, then $B = \{X_{i,j}|1\leq i\leq m, 1\leq j\leq k\}$. Hence, by selecting at most $k$ elements from $B$, we have $H(\Ad{m, k}) = k$. On the other hand, the set of $k$ elements that maximizes the entropy is $\{Y_1, \dots, Y_m\}$. Note that $H(Y_i)=k$ and $Y_i \perp Y_j$ for $i\neq  j$. Hence, $H(A^c) = k\cdot m$ if $m\geq k$ or otherwise $H(\Ac{k}) = k^2$.

\subsection{Proof of Theorem~\ref{th:ggreedy}}

Let us first mention a slight generalization over the performance of the standard greedy algorithm. It follows easily from the argument in \citep{nemhauser78}. 
\begin{lemma}\label{lem:ggeneral}
Let $f$ be a non-negative submodular function, and let $\Agc{q}$ of cardinality $q$ be the greedy selected set by the standard greedy algorithm. Then, $$f(\Agc{q})\geq \left(1-e^{-\frac{q}{k}}\right) f(\Ac{k}). $$ 
\end{lemma}

\noindent  By Lemma~\ref{lem:ggeneral} we know that $$f(\Aigc[\kappa])\geq (1-\exp(-\kappa/k)) f(\Aic[k]).$$ Now, let us define 
\begin{eqnarray*}
\Bgc &=& \cup_{i=1}^m \Aigc[\kappa], \\
\Oa{\kappa} &=& \max_i f(\Aigc[\kappa]),\\
\tilde{A}[\kappa] &=& {\arg\max}_{S\subseteq \Bgc \& |S|\leq\kappa} f(S).
\end{eqnarray*}
%
 Then by using Lemma~\ref{lem:ggeneral} again, we obtain 
\begin{eqnarray*}
f(\Agd{m,\kappa})&\geq& \max \big \{ f(\Oa{\kappa}), ~(1-\exp(-\kappa/\kappa)) f(\tilde{A}[\kappa]) \big \} \\
&\geq& \frac{(1-\exp(-\kappa/k))}{\min(m,k)}f(\Ac{k}). 
\end{eqnarray*}

\subsection{Proof of Proposition~\ref{prop:lipschitz-active}}
Let $K$ be a positive definite kernel matrix defined in section \ref{sec:nonparam}.  If we replace a point $e_i \in S$ with another point $e'_i \in V \setminus S$, the corresponding row and column $i$ in the modified kernel matrix $K^\prime$ 
 will be changed. W.l.o.g assume that we replace the first element $e_1 \in S$ with another element $e'_1 \in V \setminus S$, i.e., $\Delta K = 
  K' - K $ has the following form with non-zero entries only on the first row and first column, 
\[
\Delta K \equiv  K' - K \leq 
\begin{pmatrix}
  a_1 & a_{2} & \cdots & a_{k} \\
  a_{2} & 0 & \cdots & 0\\
  \vdots  & \vdots  & \ddots & \vdots  \\
  a_{k} & 0 & \cdots & 0
 \end{pmatrix}.
\]

\noindent Note that kernel is Lipschitz continuous with constant $\mathcal{L}$, hence we have $|a_i| \leq \mathcal{L}d(e_1,e'_1)$ for $1 \leq i \leq k$.
Then the absolute value of the change in the objective function would be
\begin{align} \label{eq:lipact}
 \dst\left| f(S) - f(S^\prime) \right| 
 &= \left | \frac{1}{2} \log \det(\mathbf{I}+ K^\prime) - \frac{1}{2} \log \det(\mathbf{I}+ K) \right | \notag\\
 &= \frac{1}{2} \left | \log \frac{\det(\mathbf{I}+K^\prime)}{\det(\mathbf{I}+K)} \right |  \notag\\
 & = \frac{1}{2} \left | \log \frac{\det (\mathbf{I}+K + \Delta K)}{\det(\mathbf{I}+K)} \right | \notag\\
 &= \frac{1}{2} \left | \log [ \det(\mathbf{I}+K + \Delta K) . \det(\mathbf{I}+K)^{-1} ] \right | \notag\\
 & = \frac{1}{2} \left | \log \det(\mathbf{I} + \Delta K (\mathbf{I} + K)^{-1}) \right|.
\end{align}

\noindent Note that since $K$ is positive-definite, $\mathbf{I} + K$ is an invertible matrix. Furthermore, since $\Delta K$ and $K$ are symmetric matrices they both have $k$ real eigenvalues.
Therefore, $(\mathbf{I} + K)^{-1}$ has $k$ eigenvalues $\lambda_i = \frac{1}{1+\lambda'_i}\leq 1$, for $1 \leq i \leq k$, where $\lambda'_1 \cdots \lambda'_k$ are (non-negative) eigenvalues of kernel matrix $K$.

Now, we bound the maximum eigenvalues of $\Delta K$ and $\Delta K (\mathbf{I} + K)^{-1}$ respectively.
Consider vectors $x, x' \in \mathbb{R}^n$, such that $||x||_2 = ||x'||_2 = 1$. We have,
\begin{align} \label{eq:eig}
 \left| x^T \Delta K ~ x' \right| &= \left|
\begin{pmatrix} \notag\\
  x_{1}  \\
  x_{2} \\
  \vdots\\
  x_{k} 
 \end{pmatrix}^T 
\begin{pmatrix}
  a_1 & a_{2} & \cdots & a_{k} \\
  a_{2} & 0 & \cdots & 0\\
  \vdots  & \vdots  & \ddots & \vdots  \\
  a_{k} & 0 & \cdots & 0
 \end{pmatrix}
\begin{pmatrix}
  x'_{1}  \\
  x'_{2} \\
  \vdots\\
  x'_{k} 
 \end{pmatrix}
\right | \notag\\
& = \left | 
\begin{pmatrix}
  x_{1}  \\
  x_{2} \\
  \vdots\\
  x_{k} 
 \end{pmatrix}^T
 \begin{pmatrix}
  \sum_{i=1}^k a_i x'_i  \\
  a_{2} x'_1 \\
  \vdots\\
  a_{k} x'_1 
 \end{pmatrix} \right | \notag\\
 & = \left| x_1 \sum_{i=1}^k a_i x'_i + x'_1 \sum_{i=2}^k a_i x_i \right| \notag\\
 & = |x_1|.  \left|\sum_{i=1}^k a_i x'_i \right| + |x'_1| . \left|\sum_{i=2}^k a_i x_i \right| \notag\\
 & = |x_1|.  \sum_{i=1}^k |a_i x'_i | + |x'_1| . \sum_{i=2}^k |a_i x_i| \notag\\
 & \leq 2k \mathcal{L} d(e_1,e'_1),
 \end{align} 

\noindent where we used the following facts to derive the last inequality: 1) the Lipschitz continuity of the kernel gives us an upperbound on the values of $|a_i|$, i.e., $|a_i| \leq \mathcal{L}d(e_1,e'_1)$ for $1 \leq i \leq k$; and 2) since $||x||_2 = ||x'||_2 = 1$, the absolute value of the elements in vectors $x$ and $x'$ cannot be greater than 1, i.e., $|x_i| \leq 1, ~|x'_i| \leq 1$, for $1\leq i \leq k$. 
%
 Therefore, 
$$\lambda_{\max} (\Delta K) =  \max_{x:~||x||_2 = 1} |x^T \Delta K x| \leq 2k \mathcal{L} d(e_1,e'_1).$$ 

Now, let $v_1, \cdots v_k \in \mathbb{R}^n$ be the $k$ eigenvectors of matrix $(\mathbf{I} + K)^{-1}$. Note that $\{ v_1, \cdots v_k \}$ is an orthonormal system and thus for any $x \in \mathbb{R}^n$ we can write it as $x = \sum_{i=1}^k c_i v_i$, and we have $||x||_2^2 = \sum_{i=1}^k c_i^2$. In order to bound the largest eigenvalue of $\Delta K (\mathbf{I} + K)$, we write
\begin{align*}
\left| x^T \Delta K ~ (\mathbf{I} + K)^{-1}  x \right| 
& = 
\left| x^T \Delta K ~ (\mathbf{I} + K)^{-1} \sum_{i=1}^k c_i v_i \right| \\ \nonumber 
& = \left| x^T \Delta K ~ \sum_{i=1}^k \lambda_i c_i v_i \right| \\ \nonumber
& = \left| \left(\sum_{j=1}^k c_jv_j \right)^T \Delta K \left(\sum_{i=1}^k \lambda_i c_iv_i \right) \right| \\ \nonumber
& = \left| \sum_{i,j=1}^k \lambda_i c_i c_j v_j^T \Delta K v_i \right| \\
 & \stackrel{(a)}{\leq} 2k \mathcal{L} d(e_1,e'_1) \sum_{i,j=1}^k |c_i| |c_j| \\ \nonumber
 & = 2 k \mathcal{L} d(e_1,e'_1)   \left( \sum_{i=1}^k |c_i| \right)^2, \\
\end{align*}
where in (a) we used Eq. \ref{eq:eig} and the fact that $\lambda_i \leq 1$ for $1\leq i \leq k$. Using Cauchy-Schwarz inequality  
$$\left(\sum_{i=1}^k |c_i| \right)^2 \leq k \sum_{i=1}^k |c_i|^2$$
and the assumption $||x||_2 = 1$, we conclude
\begin{align*}
\left| x^T \Delta K ~ (\mathbf{I} + K)^{-1}  x \right|
& \leq 2  k^2 \mathcal{L} d(e_1,e'_1)   \sum_{i=1}^k \left| c_i^2 \right| \\ \nonumber
& \leq 2  k^2 ||x||_2^2 \mathcal{L} d(e_1,e'_1) \\ \nonumber
& \leq 2 k^2 \mathcal{L} d(e_1,e'_1) .
\end{align*}

\noindent Therefore, 
\begin{equation}\label{eq:det}
 \lambda_{\max} \left(\Delta K (\mathbf{I} + K)^{-1} \right) =  \max_{x:~||x||_2=1} \left| x^T \Delta K ~ (\mathbf{I} + K)^{-1}  x \right| \leq 2k^2 \mathcal{L} d(e_1,e'_1).
\end{equation}

Finally, we can write the determinant of a matrix as the product of its eigenvalues, i.e. 
\begin{equation} \label{eq:eig2}
\det( \mathbf{I} + \Delta K ( \mathbf{I} + K)^{-1}) \leq (1+2 k^2 \mathcal{L} d(e_1,e'_1))^k.
\end{equation}

By substituting Eq. \ref{eq:det} and Eq. \ref{eq:eig2} into Eq. \ref{eq:lipact} we obtain
\begin{align*}
\dst\left| f(S) - f(S^\prime)\right|  
& \leq \frac{1}{2} \left|\log (1+2 k^2 \mathcal{L} d(e_1,e'_1))^k \right| \\ \nonumber 
& \leq \frac{k}{2} \left|\log (1+2 k^2 \mathcal{L} d(e_1,e'_1)) \right| \\ \nonumber 
& \leq  k^3 \mathcal{L} d(e_1,e'_1),
\end{align*}
where in the last inequality we used  $\log(1+x) \leq x$, for $x \geq 0$.

\noindent Replacing all the $k$ points in set $S$ with another set $S'$ of the same size, we get
\[
\dst\left| f(S) - f(S^\prime)\right| \leq  k^3 \mathcal{L} \sum_{i=1}^k d(e_i,e'_i).
\]
Hence, the differential entropy of the Gaussian process is $\lambda$-Lipschitz with  $\lambda =   \mathcal{L} k^3$.

\subsection{Proof of Proposition~\ref{prop:lipschitz-exemplar}}
Assume we have a set $S$ of $k$ exemplars, i.e., $S_0=\{ e_1, \cdots , e_k \}$, and each element of the dataset $v \in V$ is assigned to its closest exemplar. Now, if we replace set $S$ with another set $S^\prime$ of the same size, the loss associated with every element $v \in V$ may be changed. W.l.o.g, assume we swap one exemplar at a time, i.e., in step $i, 1 \leq i \leq k$, we have $S_i=\{ e'_1, \cdots, e'_i, e_{i+1}, \cdots , e_k \}$. Swapping the $i^{th}$ exemplar $e_i \in S_{i-1}$ with another element $e'_i \in S'$, 4 cases may happen: 1) element $v$ was not assigned to $e_i$ before and doesn't get assigned to $e'_i$, 2) element $v$ was assigned to $e_i$ before and gets assigned to $e'_i$, 3) element $v$ was not assigned to $e_i$ before and gets assigned to $e'_i$, 4) element $v$ was assigned to $e_i$ before and gets assigned to another exemplar $ e_x \in S_i \setminus \{e'_i\}$. 
%
For any element $v \in V$, we look into the four cases and show that in each case 
$$| l(e'_i,v) - l(e_i,v) | \leq  d(e_i, e'_i) ~ \alpha R^{\alpha-1}. $$ 

\begin{itemize}
\item
Case 1: In this case, element $v$ was assigned to another exemplar $e_x \in S_i \setminus {e_i}$ and the assignment doesn't change. Therefore, 
there is no change in the value of the loss function.  
\item
Case 2: In this case, element $v$ was assigned to $e_i$ before and gets assigned to $e'_i$.
let $a =  d(e_i,v)$ and $b = d(e'_i,v)$. Then we can write
\begin{align}\label{eq:eq}
| l(e'_i,v) - l(e_i,v) | &= | a^\alpha - b^\alpha | \notag \\
 & = |(a-b)|(a^{\alpha-1}+a^{\alpha-2}b + \cdots + ab^{\alpha-2} + b^{\alpha-1}) \notag \\
 & \leq d(e_i,e'_i) ~ \alpha R^{\alpha-1},
\end{align} 

\noindent where in the last step we used triangle inequality 
$|d(e'_t,v)-d(e_t,v)| \leq d(e_t,e'_t)$
and the fact that data points are in a ball of diameter $R$ in the metric space.

\item
Case 3: In this case, $v$ was assigned to another exemplar $e_x \in S_{i-1} \setminus \{ e_i \}$ and gets assigned to $e'_i$,  which implies that 
$| l(e'_i,v) - l(e_x,v) | \leq | l(e_i,v) - l(e'_i,v) |$, 
since otherwise $e$ would have been assigned to $e_t$ before. 
%

\item
Case 4: In the last case, element $v$ was assigned to $e_i$ before and gets assigned to another exemplar $ e_x \in S_i \setminus \{e'_i\}$. Thus,
we have 
$| l(e_x,v) - l(e_i,v) | \leq | l(e'_i,v) - l(e_i,v) |$
since otherwise $v$ would have been assigned to $e_x$ before. 
Hence, in all four cases the following inequality holds:
$$| \min_{e \in S_{i-1}} l(e,\upsilon) - \min_{e \in S_i} l(e,\upsilon) | \leq | l(e'_i,v) - l(e_i,v) | \leq d(e_i,e'_i) ~ \alpha R^{\alpha-1}.$$

\end{itemize}

\noindent By using Eq. \ref{eq:eq} and averaging over all elements $v \in V$, we have 
\begin{align*}
\left|L(S_{i-1}) - L(S_i) \right| 
& = \frac{1}{|V|}\sum_{v\in V} | \min_{e \in S_{i-1}} l(e,\upsilon) - \min_{e \in S_i} l(e,\upsilon) | \\ \nonumber
& \leq \alpha R^{\alpha-1} d(e_i,e'_i).
\end{align*}
Thus, for any point $e_0$ that satisfies 
$$\max_{v' \in V}l(v,v') \leq l(v,e_0), \quad \forall v \in V \setminus S,$$
we have $L(\{ e_0 \cup S \}) = L(\{ S \})$ and thus
\begin{align*}
\left| f(S_{i-1}) - f(S_i) \right| 
& = \left| L(\{ e_0 \}) - L(\{ e_0 \cup S_{i-1} \}) - L(\{ e_0 \}) + L(\{ e_0 \cup S_i\}) \right| \\ \nonumber
& \leq  \alpha R^{\alpha-1} d(e_i,e'_i).
\end{align*}
\noindent Now, if we replace all the $k$ points in set $S$ with another set $S'$ of the same size, we get
\begin{align*}
\left |f(S) - f(S') \right| =  
& \left | \sum_{i=1}^k  f(S_{i-1}) - f(S_i) \right | \\ \nonumber
& = \sum_{i=1}^k \left |f(S_{i-1}) - f(S_i) \right| \\ \nonumber
& \leq \alpha  R^{\alpha-1} \sum_{i=1}^k d(e_i,e'_i). 
\end{align*}

\noindent Therefore, for $l = d^\alpha$, the loss function is $\lambda$-Lipschitz  with   $\lambda = \alpha R^{\alpha - 1}$.

\subsection{Proof of Theorem~\ref{thm:alg-neighborhoods}}
In the following, we say that sets $S$ and $S'$ are $\gamma$-close if $|f(S)-f(S')|\leq \gamma$.
First, we need the following lemma. 

\begin{lemma}\label{thm:random-close}
If for each $e_i\in \Ac{k}, \abs{N_\alpha(e_i)}\geq km\log{(k/\delta^{1/m})}$, and if $V$ is partitioned into  sets $V_1, V_2,\dots V_m$, where each element is randomly assigned to one set with equal probabilities, then there is at least one partition with a subset $\Aic[k]$ such that $\abs{f(\Ac{k})-f(\Aic[k])}\leq \lambda\alpha k$ with probability at least $(1 - \delta)$. 
\end{lemma}

\begin{proof}
By the hypothesis, the $\alpha$ neighborhood of each element in $\Ac{k}$ contains at least $km\log{(k/\delta^{1/m})}$ elements. 
For each $e_i\in \Ac{k}$, let us take a set of $m\log{(k/\delta^{1/m})}$ elements from its $\alpha$-neighborhood. These sets can be constructed to be mutually disjoint, since each $\alpha$-neighborhood contains $m\log{(k/\delta^{1/m})}$ elements. We wish to show that at least one of the $m$ partitions of $V$ contains elements from $\alpha$-neighborhoods of each element.

Each of the $m\log{(k/\delta^{1/m})}$ elements goes into a particular $V_j$ with a probability $1/m$. The probability that a particular $V_j$ does not contain an element $\alpha$-close to $e_i\in \Ac{k}$ is $\dst \frac{\delta^{1/m}}{k}$.  The probability that $V_j$ does not contain elements $\alpha$-close to one or more of the $k$ elements is at most $\delta^{1/m}$ (by union bound). The probability that {\em each} $V_1, V_2,\dots V_m$ does not contain elements from the $\alpha$-neighborhood of one or more of the $k$ elements is at most $\delta$. Thus, with high probability of at least $(1 - \delta)$, at least one of $V_1, V_2,\dots V_m$ contains an $\Aic[k]$ that is $\lambda\alpha k$-close to $\Ac{k}$.
\end{proof}

By lemma~\ref{thm:random-close}, for some $V_i$, $\abs{f(\Ac{k})-f(\Aic[k]})\leq \lambda\alpha k$ with the given probability. Furthermore, $f(A^{gc}_i[\kappa])\geq(1-e^{-\kappa/k})f(\Aic[k]) $ by Lemma~\ref{lem:ggeneral}. Therefore, the result follows using arguments analogous to the proof of Theorem~\ref{th:ggreedy}.


\subsection{Proof of Theorem~\ref{thm:alg-sampling}}

The following lemma says that in a sample drawn from distribution over an infinite dataset, a sufficiently large sample size guarantees a dense neighborhood near each element of $\Ac{k}$ when the elements are from representative regions of the data.

\begin{lemma}\label{thm:sampling-count}
A number of elements: $\dst n\geq \frac{8km\log{(k/\delta^{1/m})}}{\beta  g(\alpha)}$, where $\alpha\leq \alpha^*$, suffices to have at least $4km\log{(k/\delta^{1/m})}$ elements in the $\alpha$-neighborhood of each $e_i\in \Ac{k}$ with probability at least $(1-\delta)$, for small values of $\delta$. 
\end {lemma}
\begin{proof}
The expected number of $\alpha$-neighbors of an $e_i\in \Ac{k}$, is $E[\abs{N_{\alpha}(e_i)}]\geq 8km\log{(k/\delta^{1/m})}$. We now show that in a random set of samples, at least a half of this number of neighbors is realized with high probability near each element of $\Ac{k}$. 

This follows from a Chernoff bound:
\[\dst P[\abs{N_{\alpha}(e_i)}\leq 4km\log{(k/\delta^{1/m})}]\leq e^{-km\log{(k/\delta^{1/m})}}\leq (\delta^{1/m}/k)^{km}.\]

\noindent Therefore, the probability that some $e_i\in \Ac{k}$ does not have a suitable sized neighborhood is at most $k(\delta^{1/m}/k)^{km}$. For $\delta\leq 1/k$, $k\delta^{km}\leq \delta^m$. Therefore, with probability at least $(1-\delta)$, the $\alpha$-neighborhood of each element $e_i\in \Ac{k}$ contains at least $4km\log{(1/\delta)}$ elements.
\end{proof}

\begin{lemma}\label{thm:epsilon-sampling} 
For $\dst n\geq \frac{8km\log(k/\delta^{1/m})}{\beta  g(\frac{\eps}{\lambda k})}$, where $\frac{\eps}{\lambda k}\leq \alpha^*$, if $V$ is partitioned into  sets $V_1, V_2,\dots V_m$, where each element is randomly assigned to one set with equal probabilities, then for sufficiently small values of $\delta$, there is at least one partition with a subset $\Aic[k]$ such that $\abs{f(\Ac{k})-f(\Aic[k])}\leq \eps$ with probability at least $(1-\delta)$. 
\end{lemma}

\begin{proof}
Follows directly by combining Lemma~\ref{thm:sampling-count} and Lemma~\ref{thm:random-close}. The probability that some element does not have a sufficiently dense $\eps/{\lambda k}$-neighborhood with $km\log(2k/\delta^{1/m})$ elements is at most $(\delta/2)$ for sufficiently small $\delta$, and the probability that some partition does not contain elements from the one or more of the dense neighborhoods is at most $(\delta/2)$. Therefore, the result holds with probability at least $(1-\delta)$.
\end{proof}

By Lemma~\ref{thm:epsilon-sampling}, there is at least one $V_i$ such that $\abs{f(\Ac{k})-f(\Aic[k])}\leq \eps$ with the given probability. And $f(A^{gd}_i[\kappa])\geq(1-e^{-\kappa/k})f(\Aic[k]) $ using Lemma~\ref{lem:ggeneral}. The result follows using arguments analogous to the proof of Theorem~\ref{th:ggreedy}.


\subsection{Proof of Theorem~\ref{th:decom}}
Note that each machine has on the average $n/m$ elements. Let us define $\Pi_i$ the event that $n/2m<|V_i|< 2n/m$. Then based on the Chernoff bound we know that $\Pr(\neg\Pi_i)\leq 2 \exp(-n/8m)$. Let us also define $\xi_i(S)$ the event that $|f_{V_i}(S) - f(S)|<\epsilon$, for some fixed $\epsilon<1$ and a fixed set $S$ with $|S|\leq k$. Note that $\xi_i(S)$ denotes the event that the empirical mean is close to the true mean. Based on the Hoeffding inequality (without replacement) we have $ \Pr(\neq \xi_i{S} | \leq 2\exp(-2n\epsilon^2/m)$. Hence, $$\Pr(\xi_i(S)\wedge \Pi_i) \geq 1- 2\exp(-2n\epsilon^2/m)- 2 \exp(-n/8m).$$
Let $\xi_i$ be an event that $|f_{V_i}(S) - f(S)|<\epsilon$, for any $S$ such that $|S|\leq \kappa$. Note that there are at most $n^\kappa$ sets of size at most $\kappa$. Hence, 
\begin{equation}\label{eq:event}
\Pr(\xi_i\wedge \Pi_i) \geq 1- 2 n^\kappa(\exp(-2n\epsilon^2/m)-  \exp(-n/8m)).
 \end{equation}
 As a result, for $\epsilon<1/4$ we have
  $$\Pr(\xi_i\wedge \Pi_i) \geq 1- 4 n^\kappa \exp(-2n\epsilon^2/m).$$
  
\noindent   Since there are $m$ machines, by the union bound we can conclude that 
   $$\Pr((\xi_i\wedge \Pi_i)\text{ on all machines}) \geq 1- 4 mn^\kappa \exp(-2n\epsilon^2/m).$$
The above calculation implies that we need to choose $\delta\geq 4 mn^\kappa \exp(-2n\epsilon^2/m)$. Let $n_0$ be chosen in a way that for any $n\geq n_0$ we have $\ln(n)/n\leq \epsilon^2/(mk)$. Then,  we need to choose $n$ as follows:
$$n = \max\left(n_0, \frac{m\log(\delta/4m)}{\epsilon^2}\right).$$

\noindent Hence for the above choice of $n$,  there is at least one $V_i$ such that $\abs{f(\Ac{k})-f(\Aic[\kappa])}\leq \eps$ with probability $1-\delta$. Hence the solution is $\epsilon$ away from the optimum solution with probability $1-\delta$. Now if we confine the evaluation of $f(\Aic)$ to  data point only  in machine $i$ then under the assumption of Theorem~\ref{thm:alg-sampling} we lose another $\epsilon$. Formally, the result at this point simply follows by combining Theorem~\ref{th:ggreedy} and Theorem \ref{thm:alg-sampling}. 

\subsection{Proof of Theorem~\ref{th:ggreedy_const}}

The proof is similar to the proof of Theorem \ref{th:gap} and Theorem \ref{th:ggreedy} and follows from the following lemmas.
\begin{lemma}\label{lemm_kz}
$ \max_i f(\Aic[\zeta]) \geq \frac{1}{m} f(\Ac{\zeta}). $
\end{lemma}

\begin{proof}
Let $B_i$ be the elements in $V_i$ that are contained in the optimal solution, $B_i = \Ac{\zeta} \cap V_i$. Since $\Ac{\zeta} \in \zeta$ and $\zeta$ is a set of hereditary constraints, we must have $B_i \in \zeta$ as well.
Using submodularity of $f$ and by the same argument as in the proof of Lemma \ref{lemm_m}, we have
\begin{eqnarray}
f(\Ac{\zeta}) = f(B_1 \cup \cdots \cup B_m) \nonumber
&=& f(B_1) + f(B_2|B_1) + \cdots + f(B_m | B_{m-1} , \cdots, B_1) \\\nonumber
&\leq & f(B_1) + \cdots + f(B_m).\nonumber
\end{eqnarray}
Since $f(\Aic[\zeta]) \geq f(B_i)$ we get
$$ f(\Ac{\zeta}) \leq  f(A_1^c[\zeta]) + \cdots + f(A^c_m[\zeta])
\leq m \max_i f(\Aic{[\zeta]}).$$
\end{proof}
\begin{lemma}\label{lemm_mz}
$ \max_i f(\Aic[\zeta]) \geq \frac{1}{k} f(\Ac{\zeta}). $
\end{lemma}

\begin{proof}
The proof follows the outline of the proof of Lemma \ref{lemm_k}.
Let $f(\Ac{\zeta}) = f(\{u_1, \cdots, u_{\rho([\zeta])} \})$. 
Since $\Ac{\zeta} \in \zeta$ and $\zeta$ is a set of hereditary constraints, we have $u_i \in \zeta$. Using submodularity of $f$, we have
$$ f(\Ac{\zeta}) \leq \sum_{i=1}^{\rho([\zeta])} f(u_i) \leq \rho([\zeta]) f(u^*).$$
where $u^* = \arg \max_i f(u_i).$ 
Suppose that $u^* \in V_j$, we get
$$f(\max_i f(\Aic[\zeta])) \geq f(A_j^c[\zeta]) \geq f(u^*) \geq \frac{1}{\rho([\zeta])} f(\Ac{\zeta}).$$
\end{proof}
Since $f(\Ad{m, \rho([\zeta])})\geq \max_i f(\Aic[\zeta])$;
from Lemma \ref{lemm_mz} and \ref{lemm_kz} we have 
\begin{equation}\label{eq:xappr}
f(\Ad{m, \rho([\zeta])}) \geq \frac{1}{\min(m,\rho([\zeta]))} f(\Ac{\zeta}).
\end{equation}
For the black box algorithm \BBalg ~with a $\tau$-approximation guarantee, we have $$f(A_i^\BBalg[\zeta]) \geq \tau f(\Aic[\zeta]).$$
Now, we generalize the definitions used in the proof of Theorem \ref{th:ggreedy}
\begin{eqnarray*}
\Bgc &=& \cup_{i=1}^m \Aigc[\zeta], \\
\Oa{\zeta} &=& \max_i f(\Aigc[\zeta]),\\
\tilde{A}[\zeta] &=& {\arg\max}_{S\subseteq \Bgc \& |S|\leq\rho([\zeta])} f(S).
\end{eqnarray*} 
Then using Eq. \ref{eq:xappr} again, we obtain
\begin{eqnarray*}
f(\Agd{m,\zeta})&\geq& \max \big \{ f(\Oa{\zeta}), ~\tau f(\tilde{A}[\zeta]) \big \} \\
&\geq& \frac{\tau}{\min(m,\rho([\zeta]))}f(\Ac{\zeta}). 
\end{eqnarray*}

Note that since we do not use monotonicity of the submodular function in any of the proofs, the results hold in general for constrained maximization of any non-negative submodular function.

\subsection{Proof of Theorem~\ref{thm:BBalg-neighborhoods}}
\begin{lemma}\label{thm:random-closeB}
If for each $e_i\in A^\BBalg[\zeta], \abs{N_\alpha(e_i)}\geq \rho([\zeta]) m\log{(\rho([\zeta])/\delta^{1/m})}$, and if $V$ is partitioned into  sets $V_1, V_2,\dots V_m$, where each element is randomly assigned to one set with equal probabilities, then there is at least one partition with a subset $A_i^\BBalg[\zeta] \in \zeta$ such that \\
$\abs{f(\Ac{\zeta})-f(A_i^\BBalg[\zeta])}\leq \lambda\alpha \rho([\zeta])$ with probability at least $(1 - \delta)$. 
\end{lemma}

The proof is similar to the proof of Lemma \ref{thm:random-close} by taking disjoint sets of size $m\log{(\rho([\zeta])/\delta^{1/m})}$ in an $\alpha$-neighborhood of each $e_i \in \Ac{\zeta}$ and showing that with high probability, at least one of the $m$ partitions of $V$ contains elements  
from $\alpha$-neighborhoods of each element in the optimal solution. Note that now the size of the optimal solution is at most $\rho([\zeta])$. Since 
$\zeta$ is locally replaceable with parameter $\alpha$, as elements of $\Ac{\zeta}$ gets replaced by nearby elements in their $\alpha$-neighborhood, the resulting set is also a feasible solution.

By Lemma~\ref{thm:random-closeB}, for some $V_i$, $\abs{f(\Ac{\zeta})-f(A_i^\BBalg[\zeta])}\leq \lambda\alpha \rho([\zeta])$ with the given probability. On the other hand, for the black box algorithm \BBalg, we have $f(A^{\BBalg}_i[\zeta])\geq \tau f(\Aic[\zeta]) $. Therefore, the result follows using arguments analogous to the proof of Theorem~\ref{th:ggreedy_const}.

\subsection{Proof of Theorem~\ref{thm:BBalg-sampling}}

We use the following Lemmas to show that in a sample drawn from a ddistribution over an infinite dataset, a sufficiently large sample size guarantees a dense neighborhood near each element of the optimal solution.

\begin{lemma}\label{thm:sampling-count-const}
A number of elements: $\dst n\geq \frac{8 \rho([\zeta]) m\log{(\rho([\zeta])/\delta^{1/m})}}{\beta  g(\alpha)}$, where $\alpha\leq \alpha^*$, suffices to have at least $4 \rho([\zeta]) m\log{(\rho([\zeta])/\delta^{1/m})}$ elements in the $\alpha$-neighborhood of each $e_i\in \Ac{\zeta}$ with probability at least $(1-\delta)$, for small values of $\delta$. 
\end {lemma}

\begin{lemma}\label{thm:epsilon-sampling-const} 
For $\dst n\geq \frac{8\rho([\zeta])m\log(\rho([\zeta])/\delta^{1/m})}{\beta  g(\frac{\eps}{\lambda \rho([\zeta])})}$, where $\frac{\eps}{\lambda \rho([\zeta])}\leq \alpha^*$, if $V$ is partitioned into  sets $V_1, V_2,\dots V_m$, where each element is randomly assigned to one set with equal probabilities, then for sufficiently small values of $\delta$, there is at least one partition with a subset $\Aic[\zeta]$ such that $\abs{f(\Ac{\zeta})-f(\Aic[\zeta])}\leq \eps$ with probability at least $(1-\delta)$. 
\end{lemma}

The proofs follows the same arguments as in the proof of Lemma \ref{thm:sampling-count} and \ref{thm:epsilon-sampling}. 
Recall that, by assumption $\zeta$ is locally replaceable with parameter $\alpha$. Hence, for $\varepsilon \leq \alpha \lambda \rho([\zeta])$, any set $\varepsilon$-close to the optimal solution is also a feasible solution.

By Lemma~\ref{thm:epsilon-sampling-const}, there is at least one $V_i$ such that $\abs{f(\Ac{\zeta})-f(\Aic[\zeta])}\leq \eps$ with the given probability. Furthermore, for the black box algorithm \BBalg, we have $f(A^{gd}_i[\zeta])\geq \tau f(\Aic[\zeta]) $. Thus the result follows using arguments analogous to the proof of Theorem~\ref{th:ggreedy_const}.

\subsection{Proof of Theorem~\ref{th:BBalg-decom}}
Again the proof follows the same line of reasoning as the proof of Theorem \ref{th:decom}, except that for a constraint set $\zeta$ with $\rho([\zeta]) = \max_{S \in \zeta} |S|$, there are at most $n^{\rho([\zeta])}$ feasible solutions.
Using the same definitions for $\Pi_i$ and $\mathcal{E}_i$ as in the proof of Theorem \ref{th:decom}, instead of Eq. \ref{eq:event} we get
$$\Pr(\xi_i\wedge \Pi_i) \geq 1- 2 n^{\rho([\zeta])}(\exp(-2n\epsilon^2/m)-  \exp(-n/8m)).$$ As a result, for $\epsilon<1/4$ and using union bound we conclude that  
$$\Pr((\xi_i\wedge \Pi_i)\text{ on all machines}) \geq 1- 4 mn^{\rho([\zeta])} \exp(-2n\epsilon^2/m).$$
which implies that we need to choose $\delta\geq 4 mn^\rho([\zeta]) \exp(-2n\epsilon^2/m)$.
Now if $n_0$ be chosen in a way that for any $n\geq n_0$ we have $\ln(n)/n\leq \epsilon^2/(mk)$, we get $n \geq \max (n_0, m\log(\delta/4m)/\epsilon^2)$.

Bearing in mind that $\zeta$ is locally replaceable, there is at least one $V_i$ such that 
 the solution $\Aic[\zeta]$ is feasible and $\epsilon$ away from the optimum solution with probability $1-\delta$. 
 Now under the assumption of Theorem~\ref{thm:BBalg-sampling}, if we evaluate  $f(\Aic)$ only on machine $i$, then  we lose another $\epsilon$. Now by combining Theorem~\ref{th:ggreedy_const} and Theorem \ref{thm:BBalg-sampling} we get the desired result.

\vskip 0.2in
\bibliography{mirzasoleiman14}

\end{document}